\documentclass[journal,onecolumn,12pt]{IEEEtran} 

\usepackage{cite}
\usepackage{array}
\usepackage{amsmath}
\usepackage{pgfplots}
\usepackage{amsfonts}
\usepackage{mathabx} 
\usepackage{mathdots}
\usepackage{amsthm}
\usepackage{mathrsfs}
\usepackage{amssymb}
\usepackage{enumerate}
\usepackage{url} 
\usepackage{hyperref}
\usepackage{multirow}
\usepackage{dsfont}
\usepackage{algorithm}
\usepackage{algorithmic}
\usepackage{latexsym}
\usepackage{color}
\usepackage{framed}
\usepackage{graphicx}
\usepackage{subfigure}
\usepackage{soul} 

\usepackage{hhline}

\newtheorem{proposition}{Proposition}

\newcounter{sidebar}
\definecolor{vertclair}{rgb}{0.8,0.95,0.8}
\newsavebox{\SidebarBox}

\newcommand{\tens}[1]{\boldsymbol{\mathcal{#1}}}

\newcommand{\tT}{\tens{T}}

\newcommand{\matr}[1]{\boldsymbol{#1}}
\newcommand{\mA}{\matr{A}}
\newcommand{\mB}{\matr{B}}
\newcommand{\mC}{\matr{C}}
\newcommand{\mD}{\matr{D}}
\newcommand{\mS}{\matr{S}}

\newcommand{\mM}{\matr{M}}
\newcommand{\mP}{\matr{P}}
\newcommand{\mT}{\matr{T}}
\newcommand{\mX}{\matr{X}}
\newcommand{\mU}{\matr{U}}
\newcommand{\mV}{\matr{V}}
\newcommand{\mQ}{\matr{Q}}

\newcommand{\vect}[1]{\boldsymbol{#1}}

\newcommand{\tensp}{\mathop{\otimes}}      
\newcommand{\kr}{\odot}     
\newcommand{\hadam}{\boxdot}    
\newcommand{\T}{{\sf T}}        

\makeatletter
\newcommand{\rank}[1]{\mathop{\operator@font rank}\{#1\}}
\newcommand{\colrank}[1]{\mathop{\operator@font colrank}\{#1\}}
\newcommand{\krank}[1]{\mathop{\operator@font krank}\{#1\}}
\newcommand{\trace}[1]{\mathop{\operator@font trace}\{#1\}}
\newcommand{\Diag}[1]{\mathop{\operator@font Diag}\{#1\}}    
\newcommand{\diag}[1]{\mathop{\operator@font diag}\{#1\}}    
\newcommand{\Span}[1]{\mathop{\operator@font Span}\{#1\}}    
\newcommand{\argmin}{\mathop{\operator@font argmin}}

\makeatother

\definecolor{purple}{rgb}{0.6,0,0.7}
\definecolor{Purple}{rgb}{0.7,0,0.4}
\definecolor{burgundy}{RGB}{159,29,53}
\definecolor{arylideyellow}{rgb}{0.91, 0.84, 0.42}
\definecolor{bananayellow}{rgb}{1.0, 0.88, 0.21}
\definecolor{gris25}{gray}{0.90}
\definecolor{bordure}{rgb}{0.09,0.17,0.68}
\definecolor{darkred}{rgb}{0.6,0,0}
\definecolor{darkgreen}{rgb}{0.0,0.7,0}

\makeatletter
\renewcommand{\maketag@@@}[1]{\hbox{\m@th\normalsize\normalfont#1}}%
\makeatother

\begin{document}

\title{Dictionary-based Tensor Canonical Polyadic Decomposition}

\author{~J\'er\'emy~Emile~Cohen*%
		 ~and~Nicolas~Gillis
\thanks{The authors acknowledge the support by the F.R.S.-FNRS (incentive grant for scientific 
research n$^\text{o}$ F.4501.16). NG also acknowledges the support by the ERC (starting grant n$^\text{o}$ 
679515). 

University of Mons 7000 Mons, Belgium (e-mail: firstname.lastname@umons.ac.be). 

Manuscript received XX, 2017; revised  XX, 2017.} 
}

\maketitle \begin{abstract} To ensure interpretability of extracted sources in
    tensor decomposition, we introduce in this paper a dictionary-based tensor
    canonical polyadic decomposition which enforces one factor to belong
    exactly to a known dictionary. A new formulation of sparse coding is
    proposed which enables high dimensional tensors dictionary-based canonical
    polyadic decomposition. The benefits of using a dictionary in tensor
    decomposition models are explored both in terms of parameter
    identifiability and estimation accuracy.  Performances of the proposed
    algorithms are evaluated on the
    decomposition of simulated data and the unmixing of hyperspectral
    images.   \end{abstract}

\IEEEpeerreviewmaketitle

\section{Introduction}

Given a mixture of several components, a classic problem in signal processing is to separate the contribution of each of these components using solely the information contained in the data. This problem is known as blind source separation and has been a particularly widely studied topic for the last two decades~\cite{comon2010handbook}. 
When the available data are contained in a multiway array, that is, a table of three entries or more, blind source separation has been successfully achieved by means of tensor decomposition techniques in a large variety of applications, ranging from 
telecommunications~\cite{SidiBG00:ieeesp} to chemometrics~\cite{Bro1998}, spectral unmixing~\cite{veganzones2016nonnegative}, neuroimaging~\cite{cichocki2015tensor}, social sciences~\cite{kroonenberg1983three} and machine learning~\cite{papalexakis2016tensors}.

The key concept behind tensor decomposition methods is the linearity of blocks of parameters of interest such as spectra, concentrations or time signatures with respect to experimental parameters such as wavelength, pixel index, subject index or time. 
For the canonical polyadic decomposition (CPD) model studied in this paper, each block of parameters of interest depends respectively on only one experimental parameter.

As an illustration, tensor decomposition techniques can be used to perform source separation with hyperspectral images, a task often referred to as spectral unmixing. Hyperspectral images are 2D images collected for a large number of wavelengths, and possibly along time. In this scenario, the blocks of parameters of interest are spectral signatures (characteristic responses of materials to light stimulation), abundances (relative concentrations) and time evolution. 

For most applications, tensor decomposition models are however not exactly
 following a  physical models because of their strong 
multilinearity  assumptions. As a consequence, modeling error has to be accounted for when processing the results of the source separation using tensors. This matters when the goal of the source separation is to identify the components in the mixture. The estimated parameters of interest have to be compared with benchmarks, and modeling error induces error on the estimated parameters which may in turn deteriorate identification performances. 

In this paper, we want to avoid splitting the identification procedure and the source separation procedure. It is shown that using a formalism inspired from sparse coding~\cite{olshausen1997sparse}, merging source separation and identification is not only possible but also offers advantages in term of uniqueness properties of the tensor decomposition, and may reduce estimation error on identified factors provided the \textit{a priori} information is accurate.

\section*{Outline and contributions}
In this work, after discussing notations and vocabulary, we make the following contributions:
\begin{itemize}

\item Formalize high order tensor sparse coding by modifying the usual sparse coding formulation for matrices \cite{olshausen1997sparse,elhamifar2012see}. We call the obtained tensor model dictionary canonical polyadic decomposition (DCPD). 

\item Provide tools for introducing flexibility in the DCPD formulation. This makes the DCPD model more suitable for practical problems.

\item Study the identifiability of parameters of the DCPD for matrices and higher order tensors, and study the existence of a best low rank DCPD approximation.

\item Develop greedy and continuous algorithms to compute the DCPD and its flexible variants.

\item Check the identification performances of the DCPD model on synthetic data with respect to the CPD model. We also use DCPD in the matrix case to perform spectral unmixing under pure-pixel assumption of the Urban and Terrain data sets\footnote{available at \url{http://www.agc.army.mil/}} and compare our results 
with state-of-the-art methods \cite{G14b,ND05,KSK12,GV14}. 

\end{itemize}


\section*{Notation and Vocabulary} 
Among various notation habits in the multiway array processing community, we choose to follow notation from \cite{Hack12,cohen2015notations}, as presented in Tables \ref{table1} and \ref{table2}. 
Although our results will be applicable to tensors of any order, we focus in
this paper on third order tensor in order to simplify the presentation. We call
a \textbf{third-order} real $K\times L \times M$ \textbf{tensor} $\tens{T}$ a
vector from a tensor space $(\mathds{R}^K \tensp \mathds{R}^L \tensp
\mathds{R}^M,\tensp)$ with $\tensp$ being a  tensor product.  A \textbf{three way array} is an element from $(\mathds{R}^{K\times L\times M},\tensp)$ where $\tensp$ is the outer product. In other words, the set of arrays is a tensor space when the outer product is used as the tensor product. A higher order tensor is a vector from a tensor space featuring at least three linear subspaces. 

In Table \ref{table2}, some useful properties of multilinear operators acting
on tensors are specified. Multilinear operators generalize linear operators
acting on vectors by defining linear operations on each vector space
$\mathds{R}^N$ composing the tensor space. These multilinear operators also
form a tensor space $(\mathds{R}^{R_1\times K} \tensp_{op}\mathds{R}^{R_2\times
K} \tensp_{op}\mathds{R}^{R_3\times K},\tensp_{op})$ where $\tensp_{op}$ will
be abusively denoted $\tensp$ although it is not the outer product. 
Indeed,  $\tensp_{op}$ is defined by $(\mU\tensp_{op}\mV )(\vect{a}\tensp\vect{b}) = \mU\vect{a}\tensp\mV\vect{b}$.

Given a $K \times L \times M$ tensor $\tens{T}$, its \textbf{canonical polyadic decomposition} (CPD) of rank $R$ can be written as follows: 
\begin{equation}
\tens{T} = \sum\limits_{r=1}^R \tens{D}_r, \label{CPD2}
\end{equation}
where $\tens{D}_r$ are decomposable tensors of the form $\tens{D}_r=\vect{a}_r \tensp \vect{b}_r \tensp \vect{c}_r$. The \textbf{rank} of $\tens{T}$ is the minimal value of $R$ such that (\ref{CPD2}) holds exactly, while we define the rank of a CPD model as the number of component $R$ in that model. A CPD model or a tensor is said to be a \textbf{low-rank model} or a low-rank tensor if $R$ is small with respect to all dimensions of the data or tensor.

Finding the CPD of a third-order tensor means finding rank-one tensors $\tens{D}_r$.
Yet, each tensor $\tens{D}_r$ may be defined by three vectors 
 $\vect{a}_r$, $\vect{b}_r$ and $\vect{c}_r$, only up to two scaling ambiguities; in fact, 
$\vect{a}_r \tensp \vect{b}_r \tensp \vect{c}_r = \alpha\vect{a}_r \tensp \beta\vect{b}_r \tensp \vect{c}_r/\alpha\beta$, $\forall \alpha, \beta\neq0$.

Next, it is often convenient to store these vectors in matrices called \textbf{factors} as $\matr{A}=[\vect{a}_1, \dots, \vect{a}_R]$, $\matr{B}=[\vect{b}_1, \dots, \vect{b}_R]$ and $\matr{C}=[\vect{c}_1, \dots, \vect{c}_R]$. This leads to a convenient  writing:
\begin{equation}
\tens{T} = \left( \matr{A} \tensp \matr{B} \tensp \matr{C} \right) \tens{I}_R, 
\label{CPD}
\end{equation}
where $\matr{A} \in \mathds{R}^{K\times R}$, $\matr{B} \in \mathds{R}^{L\times R}$ and $\matr{C} \in \mathds{R}^{M\times R}$ are called factor matrices, and $\tens{I}_R = \sum\limits_{r=1}^{R}{\vect{e}_i\tensp\vect{e}_i\tensp\vect{e}_i}$ belongs to $\mathds{R}^{R\times R\times R}$ with $\vect{e}_i$ a canonical basis vector of $\mathds{R}^R$, that is, $\tens{I}_R$ is a diagonal core tensor with only ones on the diagonal. A convenient interpretation of \eqref{CPD} is to see it as a change of basis. Indeed, \eqref{CPD} means that the vector $\tens{T}$ is expressed by coefficients $\tens{I}_R$ in the image of multilinear operator $\matr{A} \tensp \matr{B} \tensp \matr{C} $, with linear operators $\matr{A}$, $\matr{B}$ and $\matr{C}$ spanning respectively the first, second and third \textbf{mode} of $\tT$. Note that model \eqref{CPD} now contains $2R$ scaling indeterminacies whereas definition \eqref{CPD2} did not contain any.

Conditions on the dimensions of the tensor and rank of the decomposition are
given in the literature
\cite{kruskal1977three,domanov2013uniqueness,landsberg2013equations} to ensure
uniqueness of the factors in an unconstrained CP model, but only when noise is
absent. When these conditions are satisfied, the factors in the CPD model are
unique up to permutation ambiguity and  the  scaling described above. The model is then said to be \textbf{identifiable}. For higher order tensors, that is, arrays of three ways and more, these conditions are mild.


\begin{table}[h!]
\begin{framed}\begin{minipage}{.95\linewidth}~\vspace{-2ex} 
\begin{align*}
\mathcal{E} \tensp \mathcal{F} &  \text{: tensor product space, linear space mapped by } \tensp \text{ from } \mathcal{E} \times \mathcal{F}.  \\
\vect{a} \tensp \vect{b} ~ &  \text{: tensor product of two vectors, \emph{i.e.} an element of $\mathcal{E} \tensp \mathcal{F}$,} \\
& \text{  ~~can be understood as an outer product of vectors} \\ 
& \text{  ~~if the tensor space is an array space. \cite{Hack12} } \\
\matr{A} \kr \matr{B}  & \text{: Khatri-Rao (columnwise Kronecker) product of matrices \cite{Rao65}.}\\
\matr{A} \hadam \matr{B}  & \text{: Hadamard product of matrices, \emph{i.e.}, element-wise product \cite{Rao65}. }
\end{align*}
\end{minipage} \end{framed}
\caption{Basic definitions from linear algebra \vspace*{-1em}}
\label{table1}
\end{table}

Matrices are a particular case of tensors with only two modes. One of the main differences of the CPD for matrices compared with the CPD of higher order tensors is that it is never identifiable without additional constraints as soon as $R>1$. Note that, since the two modes of a matrix are the column space and the row space, the CPD of a matrix can be written as $\mM=\mA\mB^T$.

 A difficult problem not adressed in this manuscript is finding the rank of a
tensor, in both an exact decomposition and an approximate decomposition
scenario. A naive approach is to look at the singular value profiles of the
unfoldings, but an interested reader may refer to \cite{bro2003new,da2008robust} and references
therein. 

A tool that needs to be introduced here is the unfolding of three-way arrays. By choosing three particular ways to cut the cube of data into slices and stacking the obtained matrices, it is possible to rewrite the CPD in a matrix format. We chose the unfolding defined in \cite{cohen2015notations}, yielding the following matrix format CPD : 
\begin{equation}
\begin{array}{l}
\matr{T}_1 = \mA\left(\mB\kr\mC\right)^T,  \\
\matr{T}_2 = \mB\left(\mA\kr\mC\right)^T,  \\
\matr{T}_3 = \mC\left(\mA\kr\mB\right)^T. 
\end{array}
\end{equation}
Unfolding $\matr{T}_i$ contains all the vectors along the mode $i$. For instance, $\matr{T}_1$ contains all the $L\times M$ columns of $\tT$, thus its column space is the subspace spanned on the first mode by $\tT$, that is, the span of columns of $\mA$.


\begin{table}
\begin{framed} \begin{minipage}{.95\linewidth}~\vspace{-2ex}  
\begin{align*}
\matr{U} \tensp \matr{V} \tensp \matr{W}~~~~ &  \text{ : an operator acting on a third order tensor.} \\
\left( \matr{U} \tensp \matr{V} \tensp \matr{W} \right) \tens{T} &  \text{ :
application of } \matr{U} \text{ on the first mode, } \matr{V} \text{ on the}
\\   \text{second mode, } \matr{W} & \text{ on the third mode, also noted } \tens{T}\bullet_1 \matr{U} \bullet_2 \matr{V} \bullet_3 \matr{W}. 
\end{align*} \end{minipage} \end{framed}
\caption{Some definitions and properties of multilinear operators \vspace*{-2em}}
\label{table2}
\end{table}

\section{Dictionnary CPD : Model Definition}\label{sec-model}

An important property of low-rank models is that the computed factors bear physical meaning. Factors are interpretable when the multilinear relationship between the experimental parameters and the block of parameters of interest stems from a meaningful modeling, and when the model is identifiable. 

The first condition is met in many applications of low-rank models, for instance in fluorescence spectroscopy where spectra and concentrations depend linearly, as a first-order approximation, on the emission wavelength, the excitation wavelength and the mixture index~\cite{valeur2012molecular}. When mining a collection of text documents, low-rank factorization techniques can identify characteristic words and documents~\cite{lee1999learning}. 

The identifiability condition may however not be verified in practice. For matrices, it is indeed well known that identifiability is never achieved without imposing additional constraints for $R$ greater than one, since 
\begin{equation}
\mM=\mA\mB^T = \mA\mP\mP^{-1}\mB^T
\label{matrix_uni}
\end{equation}  
for any invertible matrix $\mP$. In the higher-order case, conditions on the rank of the model and the dimensions of the data have been reported in the literature with some variations, but are mild enough to ensure identifiability of the CPD in most applications~\cite{kruskal1977three,domanov2013uniqueness,landsberg2013equations}. 

If the factors are meaningful, it may be of crucial importance to identify them. In the example of fluorescence spectroscopy, once emission and excitation spectra have been extracted from the data using a CPD model, the final step is to recognize which chemical compound is present in the mixtures by matching the extracted spectra with known spectra of known chemicals. The goal of the models presented below is to merge this identification step with the source separation procedure using the \textit{a priori} information available on the factors.

\subsection{Dictionary-based CPD}

Formally, assume the following relationship between a matrix $\mD \in \mathds{R}^{L\times d}$ called the (over-complete) dictionary and factor $\mB$:
\begin{equation}
    \mB = \mD\mS, \quad ~  \|\vect{s}_i\|_{0} = 1 \text{ for } i\in \{1,\dots,R\},
\label{dicoB}
\end{equation}
where $\matr{S}\in\{0,1\}^{d\times R}$ is a binary matrix, 
$\|\vect{s}_i\|_{0}$ 
counts the number of non-zero values in the $i^{th}$ column of $\mS$, and $d$
is much larger than $R$. Here $\mS$ has exactly one $1$ in each column and is a
selection matrix that identifies, among the $d$ atoms of the dictionary, $R$
atoms present in  factor  $\mB$. 
The chosen columns of $\mD$ that constitute $\mB$ are given by the row index of all the ones in $\mS$. A priori, no restrictions on the atoms are assumed although some technical conditions to ensure identifiability in particular settings are discussed in Section~\ref{identif:sec}.

Viewing \eqref{dicoB} as a parameterization of $\mB$, a \textbf{dictionary CPD model (DCPD)} can be written as follows:
\begin{equation}
\left\{
\begin{array}{l}
\tT =  \left( \mA \tensp \mD\mS \tensp \mC \right) \tens{I}_R + \tens{E},  \\
 \|\vect{s}_i\|_{0} = 1 \text{ for } i\in \{1,\dots,R\},  ~ \mS\in\{0,1\}^{d\times R} . 
\end{array}\right.
\label{DCPD}
\end{equation}
In this model, the parameters are $\mA$, $\mS$ and $\mC$. 
Tensor $\tens{E}$ represent the noise, and typically  its entries  follow i.i.d.\@ known distributions such as a zero-mean Gaussian distribution. In that case, it is possible to derive the maximum likelihood estimator of parameters $\mA$, $\mS$ and $\mC$:
\begin{equation}
\begin{array}{l}
\underset{\mA,\mS,\mC}{\text{argmin}} \, \| \tT - \left( \mA \tensp \mD\mS \tensp \mC \right) \tens{I}_R \|_F^2 \\
\text{such that }  \|\vect{s}_i\|_{0} = 1 \text{ for } i\in \{1,\dots,R\}  \text{ and } \mS\in\{0,1\}^{d\times R}. 
\end{array}
\label{lulu} 
\end{equation} 
Of course, the DCPD does not apply in any situation where the CPD works. It is far from obvious that for a given application, a dictionary containing exactly all the factors on one mode is available. Even if such a dictionary is available, it is likely that some variability has to be accounted for between the atoms and the factors. Below, we try to tackle respectively cases where the relation \eqref{dicoB} should not be exact, and cases where no dictionary is available.

\subsection{Accounting for Variability} 

In spectral unmixing, two different approaches have been studied in the
literature: the fully blind unmixing and the semi-blind approach, using already
known spectra to compute regression \cite{Jose12}. However in this particular application, it is well known that the same material emits a slightly different spectrum depending on additional parameters not accounted for in a low-rank model. These variations are called spectral variability and are a major issue to using both blind unmixing and semi-blind unmixing~\cite{zare2014endmember}. Learning from this example, we emphasize that a naive DCPD as introduced above in~\eqref{DCPD} may be viewed as unrealistic. 

Therefore, it may be necessary to introduce flexibility in the DCPD model. A first way is to generalize the relationship between factor $\mB$ and the dictionary $\mD$:
\begin{equation}
\mB = f(\mD\mS,\theta), 
\label{flexB}
\end{equation}
for some function $f(x,\theta)$ mapping to $\mathds{R}^{L\times R}$ and where
$\theta$ is a random variable following some known probability $p(\theta)$. A
simple instance of \eqref{flexB} is obtained by setting $f(x,\theta)=x+\theta$
and $p(\theta)$ is a Gaussian distribution of zero mean and known white
covariance $\sigma_c\matr{I}_K\tensp\matr{I}_R$. Then \eqref{flexB} is a noisy
version of \eqref{dicoB} and  interpreting \eqref{flexB} as \textit{a priori}
information on $\mB$,  a maximum \textit{a posteriori} estimator yields  (MAP derivation similar to \cite{CabrCC16:tsp}):
\begin{equation}
\begin{array}{l}
\underset{\mA,\mB,\mS,\mC}{\text{argmin}} \, \| \tT - \left( \mA \tensp \mB \tensp \mC \right) \tens{I}_R \|_F^2 + \frac{1}{\sigma_c^2}\|\mB-\mD\mS\|_F^2 \\
\text{such that }  \|\vect{s}_i\|_{0} = 1 \text{ for } i\in \{1,\dots,R\}  \text{ and } \mS\in\{0,1\}^{d\times R}.  \label{flexlulu} 
\end{array} 
\end{equation} 
In the case where atoms can be grouped with sufficient group population, the
user may want to identify the columns of factor $\mB$ with a group of atoms
rather than with a single element of $\mD$. This is the method used in
\cite{guo2012comparison} for hyperspectral imaging, where the dictionary
contains multiple spectra for each class of mineral. In this scenario, a
solution is to cluster the available atoms and use the centroids as columns of
$\mD$ along with inter-class covariance. Then comparing $\mB$ to multiple
classes of atoms with known averages and covariances amounts to penalizing using the Mahalanobis distance \cite{guo2012comparison} in \eqref{flexlulu}.

\paragraph*{Dictionary-based PARAFAC2} There are other ways to account for
discrepancy between the dictionary and the true phenomena underlying the data.
In particular, it is not unreasonable that this discrepancy depends on the
other experimental parameters (that is, the first or third mode in our case).
For instance, in hyperspectral imaging, a modification of the geometry of the
ground due to weather conditions can modify the spectra. If the data is
collected along time, then the spectrum of a particular material can relate to
one atom in the dictionary but with a time-dependent discrepancy. Since
modeling variability is a topic-dependent task, we only provide one instance of
a modified DCPD that handles this task, to serve as a guide for further works. 

A well-known modified CPD model that accounts for variability is the PARAFAC2 model \cite{harshman1972parafac2,kiers1999parafac2}. Here we show that the PARAFAC2 model can be adapted incorporating dictionary information. The model is the following: 
for all positive integer $m$ smaller than $M$, 
\begin{equation}
\left\{
\begin{array}{l}
\mM_m = \mA \text{Diag}(\vect{c}_m) \mB_m^T,  \\
\mB_m = \mP_m\mQ, \\
\mP_m^T\mP_m = \matr{I}_R, \\
\mQ = \mD\mS, \\  \|\vect{s}_i\|_{0} = 1 \text{ for } i\in \{1,\dots,R\},  ~ \mS\in\{0,1\}^{d\times R}, 
\end{array}\right.
\end{equation}
where $\mM_m \in \mathds{R}^{K\times L}$ is the $m^{th}$ slice of data tensor $\tT$ along the third mode, $\vect{c}_m$ is the $m^{th}$ row of a matrix of parameters $\mC$, matrices $\mP_m$ with orthogonal columns are unknown and $\mQ$ is an unknown latent factor related to all the factors $\mB_m$. Dictionary $\mD$ has sizes $D\times d$ where $D$ can be different from $M$.

Here the three-way data is seen as a collection of matrices with shared first mode factor $\mA$ and similar factors $\mB_m = \mP_m\mQ$. The matrix $\mQ$ stands for a latent shared factor among matrices $\mB_m$. To include additional knowledge provided by the dictionary, is it assumed that this $\mQ$ follows equation \eqref{dicoB}. Because the columns of $\mP$ are orthogonal, 
the underlying hypothesis here is that all $\mB_m$ have the same covariance matrix, which is a relaxed assumption with respect to the underlying hypothesis in \eqref{DCPD}, that is, $\mB_m=\mQ$ for all $m$.

An hidden advantage of the PARAFAC2 version of DCPD is that row dimension of the dictionary can be different from the number of samples on the second mode, since $\mP_m$ does not need to be square. 
In practice this means that when using the PARAFAC2 dictionary model, the measured spectral bands can be different than the spectral bands at which atoms are sampled.

\subsection{Self-Dictionnary}

If no dictionary is available on any of the factor,  it is still possible in some specific cases to obtain the dictionary from the data. 
The topic of learning a dictionary from auxiliary data will be dealt with in future works. Rather, we focus here on the \textbf{separability} assumption~\cite{DS03, arora2012computing}, 
that is when the data itself can be used as a dictionary. In the matrix case, the separability has been well studied in the context of nonnegative matrix factorization (NMF), see next section for details. With the suggested formalism, the self-dictionary CPD for matrices is written as follows: 
\begin{equation}
    \mM = \mA \left(\mM^T \mS\right)^T,  
\label{self-dico}
\end{equation}
where ~  $\|\vect{s}_i\|_{0}=1 \text{ for }  
    i\in \{1,\dots,R\},  ~ \mS \in \{0,1\}^{d\times R}$.
A working hypothesis of this model for matrices is that the second mode factor
is a subset of the rows of $\mM$. 

A generalization of the separability assumption to higher order tensors is however not straightforward. 
Recently, an attempt was made to define ``pure slices'' \cite{ang2016non}, but
another possible generalization is obtained by supposing the columns of factor
$\mB$ are contained in all the $K\times M$ columns of the unfolding matrix
$\mT_{2}$. The drawback of this model is the possibly very large number of
correlated atoms in the obtained dictionary $\mT_{2}$. There are  various
approaches to extend the separability assumption to high-order tensors, but
this paper only deals with the case where $\mD=\mT_{2}$.
Considering other variants is a direction for further research.


\section{Related works}
\subsection{Sparse coding} 
Dictionary CPD presented in this paper can be seen as an extension of sparse coding for higher order tensors, see for instance \cite{olshausen1997sparse,elhamifar2012see}.  
In sparse coding, a dictionary $\mD$ is available, and the data is to be expressed as linear combinations of a small number of $R$ atoms:
\begin{equation}
\mM = \mD\mX,  
\end{equation}
 where $\mX$ is imposed to have at most $R$ non-zero rows. 
This constraint can be formulated using the $\ell_0/\ell_q$ pseudo norm of $\mX^T$ for any $q \geq 0$, 
where the $\ell_p/\ell_q$ norm of a matrix is defined as 
\begin{equation}
||X||_{\ell_p/\ell_q} 
= ||v||_p, \text{ where } v_i 
= ||\vect{x}_i||_q \, \forall i. 
\end{equation} 
In fact, $\mX$ has at most $R$ non-zero rows if and only if $||\mX^T||_{\ell_0/\ell_q} \leq R$. 
 
This model is equivalent to dictionary CPD for matrices by setting $\mS\mB^T=\mX$. 
Nevertheless there is a crucial difference between the two formulations: 
\begin{itemize}
\item By splitting the variable $\mX$ into two variables $\mS\mB^T$,  
the linear constraint $\mM = \mD\mX$ becomes non-convex with respect to $\mS$ or $\mB$. 
\item  The sparse coding formulation involves a large number of variables, $d\times m$ if $\mM$ is $n$ by $m$, whereas the DCPD formulation for matrices involves $(d+m)\times R$ parameters. 
\end{itemize}
In summary, sparse coding involves much more parameters but yields a convex parametrization of the data.

In practice, because of the presence of noice, the constraint $\mM = \mD\mX$ is replaced by the minimization of $||\mM - \mD\mX||$ (which is convex for any norm $||.||$).  
Since the mixed pseudo norm $\ell_{0}/\ell_{q}$ is non-convex, 
 most methods for computing sparse coding rely on convex relaxations. 
Row-sparsity of $\mX$ can be achieved in different ways; in particular using
convexifications based on the $\ell_1$ norm, e.g., 
$\ell_{1}/\ell_{2}$~\cite{elhamifar2012see, iordache2014collaborative},
$\ell_{1}/\ell_{\infty}$~\cite{EMO12}, 
or using more sophisticated models~\cite{BRRT12, GL13}. 
As far as we know, it is the first time sparse coding is tackled using the reformulation 
$\mX=\mS\mB^T$. Although the resulting problem cannot be  relaxed  easily, 
it involves significantly fewer variables hence will be applicable to
large-scale problems.

Recently, Salhoun et al.\@ suggested a direct extension of sparse coding for higher order tensors \cite{sahnoun2017simultaneous} in the particular context of harmonic retrievals, but an obvious difficulty is that the row-sparsity has to be imposed on Khatri-Rao products of factors. 

\subsection{NMF with self dictionary in spectral unmixing}  

Similarly, self dictionaries have been studied for matrices using the sparse coding formalism, by setting $\mD=\mM$. This is often called the pure-pixel assumption in spectral unmixing~\cite{Jose12}, since the self-dictionary model assumes some columns of the matrix $\mM$ are not mixtures of more than one column of $\mB$. As far as we know, there are mainly two types of approaches to tackle~\eqref{self-dico}:  
\begin{itemize}

\item \emph{Geometric} approaches that selects the atoms in the dictionary based on some geometric criteria, typically based on the volume of the convex hull of $\mM^T\mS$. These approaches include for example vertex component analysis (VCA)~\cite{ND05} and the successive projection algorithm (SPA)~\cite{MC01, RC03, CM11, GV14}. 
They are usually fast, running in $\mathcal{O}(mnR)$ operations. 
However, they do not always select atoms leading to a small data fitting term $||\mM-\mA(\mM^T\mS)^T||_F$ since, 
most of them do not take it into account directly, as they usually put an emphasis on some geometric properties of $\mM^T\mS$ (such as having a large volume). In particular, these methods are in general sensitive to outliers. 

\item \emph{Sparse regression} approaches that are based on the sparse coding reformulation of~\eqref{self-dico} 
\begin{equation*}
\begin{array}{l}
\min_{\mX \in \mathbb{R}_+^{d \times n}} ||\mM - \mX\mM||_F^2 
\\ \text{such that } \; 
\text{ $\mX$ has $r$ non-zero columns, }
\end{array} 
\end{equation*}
and achieve  column  sparsity constraints on the scores  $\mX$
 in different ways~\cite{elhamifar2012see, EMO12, BRRT12, GL13,
iordache2014collaborative}. These methods have the advantage to better
model~\eqref{self-dico} than geometric approaches as they take into account the
data fitting term explicitly. They usually provide good solutions but are
rather costly as an optimization problem in $dn$ variables must be solved,
where $n$ stands for the column dimension of $\mM$. In particular, since here $\mD = \mM^T$, we have $d=n$ hence $n^2$ variables. 
In hyperspectral unmixing, $n$ is the order of millions and these approaches are impractical. 
Hence pixels have to be selected in a preprocessing step~\cite{Gillis2016fast} (e.g., using a geometric approach).   
Moreover, the problem solved is an approximation of the original problem, which results may not be as close as desired to the solutions of the non-convex problem. 

\end{itemize}

In section \ref{sec-alg}, we describe several algorithms to tackle the proposed formulation \eqref{self-dico}. They combine the advantages of the two types of approaches described above: 
they are fast, running in $\mathcal{O}(mnR)$ operations, 
but taking explicitly the data fitting term $||\mM - \mX\mM||_F^2$ into account.

\subsection{Constrained tensor decompositions}

The DCPD can be understood as a constrained CPD model. It is similar in spirit to computing CPD when a basis of representation is known for one of the factors, which typically happens when using Tucker Decomposition \cite{de2000multilinear} or a basis of splines \cite{timmerman2002three} for compression. Other linearly constrained tensor decomposition model are obtained when the components are linearly dependent \cite{favier2014overview}, or when the factors are stuctured, e.g. Hankel or Toeplitz matrices \cite{goulart2016tensor}. 

From a constrained tensor decomposition perspective, the novelty of the present work is that the dictionary is overcomplete and is therefore not a basis of the factor space. Thus sparsity constraints are imposed on the coefficients and our approach is rather combinatorial. We have already published some preliminary results focusing on matrices and applications to spectral unmixing\cite{cohen2017new}.

\subsection{Parameterized factors in tensor decomposition}

Using a dictionary to help with recovering factor $\mB$ of the CPD is also
closely related to parameterizing the columns of that factor. Parametrization
is a viable option when an analytical formulation of the atoms of the
dictionary is possible.  The dictionary is then a continuous dictionary,
 see for instance~\cite[section IV]{domanov2016generic} and references therein. 
This  continuous parameterization may yet not always be achievable  in real-life applications. 
The motivation behind the two methods is however the same: reduce the number of degrees of freedom in the tensor decomposition model by providing a set of admissible solutions to improve estimation accuracy, and to restore identifiability in some pathological cases, some of which are discussed in the next section.

\subsection{Abundances estimation in hyperspectral unmixing}

In the context of spectral unmixing of hyperspectral images, using a known library of spectra to estimate the second factor $\mB$ is a widely studied topic. 
Matrix $\mA$ contains the abundances and refers to relative concentrations of materials on each pixel. 
Except for the methods described above, the most widely used techniques to compute $\mA$ when $\mD$ is known are, to the best of our knowledge, 
MESMA \cite{roberts1998mapping}, MELSUM \cite{combe2008analysis}, BSMA \cite{song2005spectral} and AutoMCU \cite{asner2003scale}, which all rely on a more or less exhaustive search of all combinations of atoms in $\mD$, except AutoMCU which only draws randomly a subset of possible atoms. For all possible combinations, abundances are computed, and the best abundances are those that minimize reconstruction error in addition to satisfying some interpretability criteria.

The proposed approach differs from these techniques since it merges blind source separation techniques for spectral unmixing, namely NMF, but featuring atom selection. To the best of our knowledge, our approach has not been described yet in the spectral unmixing literature, and is bound to be computationally less expensive than an exhaustive search.

\section{Identifiability}\label{identif:sec}
The following section contains partial results on the identifiability of the
DCPD model, cast in particular cases of interest.  In propositions 1, 3 and 4
below,
the assumption that no atoms are picked twice is made. This assumption is
necessary for proving the identifiability of the DCPD parameters, but should
not be necessarily imposed in the higher-order case. For instance, time series of
hyperspectral images may require multiple abundance and time components for a
single material when decomposed with the CPD model. 

\subsection{Matrix case}
It is well known that for matrices, the low rank CPD model is not identifiable because of the rotation ambiguity, see equation \eqref{matrix_uni}. However, when a dictionary is available for one of the modes, this rotation ambiguity may be fixed given some conditions on $\mD$:
\begin{proposition}\label{prop1} 
Let $\mM$ be a real $n\times m$ matrix  of rank $R$,  and let $\mD$ be a real $n\times d$
matrix with  $\text{spark}(\mD)>R$ where $\text{spark}(\mD)$ is the
minimum integer $k$ such that at least one subset of $k$ columns of $\mD$ is
rank-defficient.   
If there exist a full column-rank $\mS\in\{0,1\}^{d\times R}$ with column
sparsity set to 1,  and $\mA\in\mathds{R}^{n\times R}$ with nonzero columns 
such that $\mM = \mA(\mD\mS)^T$, then $\mS$ and $\mA$ are unique up to permutation ambiguity. 
%
\end{proposition}
\begin{proof} 
 Since $\text{spark}(\mD)>R$ and $\mS$ is full column rank, 
$\mB = \mD\mS$ has full column rank hence $\mA$ is unique  up to
permutations  in
the decomposition $\mM = \mA \mB^T$ if $\mB$ is unique  up to
permutations.
Moreover, given $\mM=\mA\mB^T$, because $\mA$ has no zero columns, the column
space of $\mB$ and the row space of $\mM$ are equal. Such a $\mB$ is built by selecting $R$ atoms in $\mD$ that
both span and belong to the row space of $\mM$. Because the span of $\mD$ is
strictly larger than $R$, there is only one such set of $R$ atoms. Thus $\mB$ is
unique up to permutation. 
\end{proof}


\subsection{CPD and DCPD uniqueness}

For all applications where the identifiability of the CPD is usually verified, a natural question to ask is whether the DCPD will automatically be identifiable or if some additional conditions need to be checked before trying to use the DCPD model. It turns out that given the uniqueness of the CPD up to scaling and permutations of the factors, the only requirement to obtain uniqueness of the DPCD is the uniqueness of the factorization $\mB=\mD\mS$, which itself is quite simple to check.


\begin{proposition}\label{prop2}
Let $\tT\in\mathds{R}^{K\times L\times M}$ be such that $\tT$ admits a unique rank $R$ CPD up to scaling and permutations. Then if $\mD$ does not contain collinear atoms and there exist $\mS$ verifying \eqref{dicoB}, the DCPD is also unique up to scaling and permutations.
\end{proposition}
\begin{proof}
Because the CPD of $\tT=\left( \mA \tensp \mB \tensp \mC \right)\tens{I}$ is
unique, the DCPD is unique if and only if $\Omega=\{\mS | \mB = \mD\mS \}$ is a
singleton. The set $\Omega$ is not empty by assumption, and if $\mS_1,\mS_2$
belong to $\Omega$, then the columns they select in $\mD$ have to be collinear
(or equal if there was no scaling ambiguity on $\mB$). By hypothesis this
implies  $\mS_1=\mS_2\Pi$, where $\Pi$ is a column permutation matrix.
\end{proof}
The existence of $\mS$ is what is really difficult to asses in practice. On the other hand, the dictionary may contain very correlated atoms, but not exactly collinear atoms, so that the DCPD as a model is identifiable in practice whenever the CPD is identifiable.

Moreover, using proposition \ref{prop1}, it is possible to derive a mild sufficient condition for uniqueness of DCPD. 

\begin{proposition}\label{prop3}
Let $\tT$ be a real $K\times L\times M$ tensor of rank $R$ and $\mD$ a real
$L\times d$ matrix. If there exist  a full column-rank
$\mS\in\{0,1\}^{d\times R}$ with column sparsity set to 1,  and
$\mA\in\mathds{R}^{K\times R}$ and $\mC\in\mathds{R}^{M\times R}$ such that
$\tT = \left( \mA \tensp \mD\mS \tensp \mC \right) \matr{I}_R$,  if
$\text{spark}(\mD)> R$  and if  $\mA\kr\mC$ is full column-rank  and $\mA$ and $\mC$ do not have zero columns, then $\mS$, $\mA$ and $\mC$ are unique up to permutation and scaling ambiguity. 
\end{proposition} 

\begin{proof}
By applying Proposition \ref{prop1} to  the rank $R$ matrix  $\matr{T}_2 = \mD\mS\left(\mA\kr
\mC\right)^T$, $\mS$ and the Khatri-Rao product $\mP=\mA\kr\mC$ are unique up
to permutation ambiguity. It remains to prove that the decomposition of $\mP$
into $\mA\kr\mC$ is unique provided that the dimensions of the problem are
fixed. Column-wise, we need to check that the 
 factorization
$\text{matr}(\vect{p}_i) = \vect{a}_i\vect{c}_i^T$ is unique. This rank-one approximation is
moreover up to scaling, and because this  fixed  matricization operator is an isomorphism, the decomposition of $\mP$ is therefore unique. 
\end{proof}
An important remark is that Proposition~\ref{prop3} is true for any order, since rank-one decompositions are always unique up to scaling. Similar discussion on identifiability when one factor is unique can be found in \cite{domanov2013study} and references therein. 
Moreover, if $\mA$ has many collinear columns as may be the case with spectra
in spectral unmixing, if $\mA$ has no zero column  and $\mC$ is full
column rank, then $\mA\kr\mC$ is full column-rank and  the DCPD model does not suffer from the rotation ambiguity
inherent to the CPD with collinear columns in factors. Colinear columns in
factors appear when multiple columns of a factor in one mode are necessary to
express the evolution of only one physically meaningful component.

\subsection{Existence of the best low-rank DCPD} 

For third-order tensors, another important advantage of DCPD is that it makes the optimization problem well-posed. 
In fact, for CPD, the optimal solution may not exist as the feasible set is open; see for example~\cite{de2008tensor}. 
\begin{proposition} 
Let $\text{spark}(\mD)>R$ and impose that $\mS$ is full column rank, then the optimal solution of~\eqref{lulu} is attained. 
\end{proposition}
Below we provide two different proofs, each shedding a different light on the existence of the best low rank DCPD approximation.
\begin{proof}[First proof]
Let us show that the set $\mathcal{E}_{\mB}= \{\mB(\mA \kr\mC) ~|~ \mA\in\mathds{R}^{K\times R}, \mC\in\mathds{R}^{M\times R}  \}$ is closed. First, notice that $\mathcal{E}_{\mB}$ is the image of $\mathcal{E}=\{ \mA\kr\mC ~|~  \mA\in\mathds{R}^{K\times R}, \mC\in\mathds{R}^{M\times R}  \}$ by a full column rank linear operator. Indeed, $\text{spark}(\mD)>R$. Then it is sufficient to show that $\mathcal{E}$ is closed, since $\mathcal{E}_{\mB}=(\mB^{\dagger})^{-1}\left( \mathcal{E} \right) \bigcap \text{col}(\mB)$, where $^{\dagger}$ is the left pseudo-inverse and $\text{col}(\mB)$ is the column space of $\mB$. Finally, since the set of rank-one matrices is closed, and the matrices in $\mathcal{E}$ are columnwise vectorized rank-one matrices, $\mathcal{E}$ is closed.
\end{proof}

\begin{proof}[Second proof]
Since $\text{spark}(\mD)>R$, for any full column rank $\mS$, $\mB = \mD \mS$ has rank $R$. Note that there are a finite number of such $\mS$. 
It remains to show that for any  $\mB$ of rank $R$, the infimum of 
\[
\inf_{\mA, \mC}  \, \| \tT - \left( \mA \tensp \mB \tensp \mC \right) \tens{I}_R \|_F^2 
\]
is attained. Without loss of generality, we can assume that $||\vect{a}_r||_2 = 1$ for $1 \leq r \leq R$ by the scaling degree of freedom of each rank-one tensor $\tens{D}_r=\vect{a}_r \tensp \vect{b}_r \tensp \vect{c}_r$. Using unfolding, the infimum of the above problem is attained if and only if the infimum of 
\[
\inf_{\mA, ||\vect{a}_r||_2 = 1 \forall r, \mC}  \, \| \matr{T}_3 - \mC\left(\mA\kr\mB\right)^T \|_F^2 
\] 
is attained. Moreover, since $\mC = 0$ is a feasible solution, we can add the constraint 
\[
\| \matr{T}_3 - \mC\left(\mA\kr\mB\right)^T \|_F  \leq \| \matr{T}_3 \|_F
\]
implying $\|  \mC \left(\mA\kr\mB\right)^T \|_F \leq 2 \| \matr{T}_3 \|_F$. 
For $\text{rank}(\mB) = R$ and $||\vect{a}_r||_2 = 1 \forall r$, we can show that 
\[
\sigma_{R} (\mA\kr\mB) \geq \sigma_{R}(\mB) > 0, 
\]
where $\sigma_{R}(\mB)$ is the $R$th singular value of $B$. 
In fact, denoting $\vect{a}^j$ the $j$th row of $\mA$, 
\begin{align*}
\sigma_{R}^2  (\mA\kr\mB) & = \min_{||x||_2 = 1} ||(\mA\kr\mB) x||_2^2  \\ 
& = \min_{||x||_2 = 1} \sum_{j} ||\mB (x \hadam \vect{a}^j)||_2^2 \\
& \geq \sigma_{R}^2(\mB) \min_{||x||_2 = 1} \sum_{j} ||x \hadam \vect{a}^j||_2^2 \\
& = \sigma_{R}^2(\mB) \min_{||x||_2 = 1} \sum_{r=1}^R |x_r| \, ||\vect{a}_r||_2^2 = \sigma_{R}^2(\mB).  
\end{align*}
Finally, $\|  \mC \|_F \leq 2 \frac{\| \matr{T}_3 \|_F}{\sigma_{\min}(\mB)}$ hence the feasible set can be reduced to a compact set hence the infimum is attained since the objective function is continuous and bounded below.  
\end{proof}


\section{Decomposition Algorithms : Greedy and non-greedy approaches}\label{sec-alg}

In the literature of sparse approximation, two families of algorithms have been studied extensively: greedy approaches based on matching pursuit, 
and continuous approaches based on convex relaxations of the $\ell_0$ pseudo norm \cite{tropp2004greed,tropp2006just}.
In the same spirit, we develop in the next two sections the two same kind of algorithms 
to attack~\eqref{DCPD}. 
We will compare these approaches in Section~\ref{exp-sec}. 

This paper does not explicitly discuss algorithms for computing the CPD itself. In DCPD, factor matrices $\mA$ and $\mC$ can be estimated using any off-the-shelf CPD algorithm assuming $\mB$ is fixed. This can be done for instance using the alternating least squares procedure in the unconstrained case, or exact non-negative least squares for non-negative CPD \cite{bro1997fast,gillis2012accelerated}. Therefore, below, only the estimation of $\mS$ and $\mB$ for fixed $\mA$ and $\mC$ is discussed\footnote{Algorithms introduced in this section are available at \url{https://jeremy-e-cohen.jimdo.com/downloads/}}.

\subsection{Greedy algorithms : matching pursuit}

Let us first provide an algorithm to compute the minimum \eqref{lulu} with respect to $\mS$. Since the set of solution for $\mS$ is discrete, the underlying optimization problem is combinatorial. On the other hand, computing the unconstrained CPD can be done efficiently using alternating least squares (ALS). We want to take advantage of both unconstrained CPD and greedy algorithms for computation efficiency when estimating respectively factors $\mA,\mB,\mC$ and the selection matrix $\mS$.  

This is what the matching pursuit-ALS (MPALS) described below does. 
The variable $\mB$ is injected in the problem to be a proxy of $\mD\mS$. The matrix $\mB$ is estimated using the least square update\footnote{The inverse is not actually computed, rather we solve the least square problem using any efficient solver. The bottleneck here is the large product $\matr{T}_2\left(\mA\kr\mC \right)$ since $R \ll \min(K,L,M)$.} 
\begin{equation}
\widehat{\mB}=\matr{T}_2\left(\mA\kr\mC \right)^\dagger=\matr{T}_2\left(\mA\kr\mC\right)\left(\mA^T\mA\hadam\mC^T\mC\right)^{-1},
\end{equation}
and $\widehat{\mS}$ is then evaluated by choosing the closest atom in $\mD$ up
to a scaling factor. Finally $\mB$ is reevaluated as
$\widehat{\mB}=\mD\widehat{\mS}$. In other words, $\mB$ is estimated through a
projected least squares update, where the projection space is the set spanned
by $\mD\mS$ for all $\mS$ with coefficients in $\{0,1\}$ and column sparsity
set to 1. 

\begin{algorithm}
\begin{algorithmic}
\STATE \textbf{INPUT: } array $\tT$, factors $\mA$ and $\mC$, dictionary $\mD$.
\STATE{\textbf{$\mB$ estimate:} $\matr{B}= \matr{T}_{2}\left(\matr{A}\kr\matr{C}\right)\left(\matr{A}^\T\matr{A}\hadam \matr{C}^\T\matr{C}\right)^{-1}$}
 \STATE{\textbf{$\matr{S}$ estimate:} }
  \FOR{$i$ from $1$ to $R$}
 \STATE{$S_{j^*i} = 1 \iff j^* = \text{argmax}_j \frac{\langle \vect{B}_i | \vect{D}_j \rangle}{\| \vect{D}_j \|}$}
  \ENDFOR
\STATE{ \textbf{$\mB$ reevaluation: } $\mB = \mD\mS$ }
\STATE{\textbf{OUTPUT: } Estimated scores $\mS$ and factor $\mB=\mD\mS$.}
 \end{algorithmic}
\caption{Matching Pursuit Alternating Least Squares}
\label{MPALS}
\end{algorithm}

Like most projected ALS algorithms such as the ALS algorithm for NMF\footnote{In Matlab, this is the `als' algorithm of `nnmf'.}~\cite{cichocki2002adaptive}, convergence of the global MPALS algorithm, that is including the alternating least squares estimation of $\mA$ and $\mC$, cannot be ensured since the cost function is not guaranteed to decrease at each step. Indeed, the least squares update of $\mB$ decreases the cost function, but the projection step increases it. 
Therefore MPALS is bound to have few provable results in term of convergence, contrary to continuous algorithms presented below. It is however a very simple algorithm to implement with no parameter to tune and it provides good results in both simulated and real data experiments reported in Section \ref{exp-sec}. In particular, we observed convergence in practice, and the cost function decreases for most iterations in simulations. 
The complexity of $\mathcal{O}(RKLM + RLd)$ operations per iteration of MPALS inside an ALS algorithm is about the same as plain ALS if the dictionary is not excessively large.

Note that the term greedy is a bit abusive since an atom chosen to belong to $\mB$ at some iteration of the global procedure may be discarded at a further iteration. 
What is greedy in MP-ALS is the procedure to choose atoms in the dictionary at each inner iteration. Also, MP-ALS can be easily adapted if the constraint on the number of elements in $\mS$ is modified to allow for more than one atom to be used to approximate the columns of $\mB$, 
thus the borrowed name matching pursuit.

MPALS can also easily be adapted to tackle other similar optimization problems. If the factors are constrained to be nonnegative, then the estimates of factors $\mA$ and $\mC$ can be obtained 
by nonnegative least squares as mentioned earlier, 
while factor $\mB$ should be nonnegative since the dictionary should be
nonnegative in this case.  Also, if no column of $\mD$ may be selected twice,
then after computing the scores $\text{max}_j \frac{\langle \vect{B}_i |
\vect{D}_j \rangle}{\| \vect{D}_j \|}$ for each column of $\mB$, the atoms are
assigned to each such column by solving an assignment problem in order to maximize the sum of the scores~\cite{kuhn1955hungarian}.

\paragraph*{A smooth version of MPALS}

MPALS features a projection of $\mB$ on the atoms of the dictionary at each outer iteration, which is a very rough way to impose the sparsity constraint on the scores $\mS$. 
To obtain a smoother optimization algorithm, 
we suggest to enforce sparsity constraints in a continuous manner using the flexible formulation of the DCPD for factor $\mB$, while using the projected factor $\mD\mS$ when estimating the other factors. 
We call this algorithm smooth MPALS (SMPALS), it is summarized in Algorithm~\ref{SMPALS}. 
This algorithm also finds the minimum of~\eqref{flexlulu} solving the flexible dictionary problem with Gaussian noise. 
MPALS is modified in a straightforward manner, by making the least squares update of $\mB$ depend on $\mS$, and not reevaluate $\mB$ after evaluating $\mS$. This means removing $\mB=\mD\mS$ in Algorithm \ref{MPALS}, and replacing $\mB$ first estimate with
\begin{equation}
\widehat{\mB} = \left(\matr{T}_2\left(\mA\kr\mC\right) + \lambda\mD\mS\right)\left(\mA^T\mA\hadam\mC^T\mC + \lambda \matr{I}_{R\times R}\right)^{-1}
\label{flex-updt}
\end{equation}
where $\lambda$ is a given parameter, set by the user at the beginning of the algorithm, meant to approach $\frac{1}{\sigma_c^2}$ in~\eqref{flexlulu}.

\begin{algorithm}
\begin{algorithmic}
\STATE \textbf{INPUT: } array $\tT$, factors $\mA$ and $\mC$, dictionary $\mD$, coupling parameter $\lambda>0$ and update rate $p > 1$.
 \STATE{\textbf{$\mB$ least squares estimate:} %
 \begin{equation*}
 \matr{B}= \left(\tens{T}_{(2)}\left(\matr{A}\kr\matr{C}\right) + \lambda\mD\mS \right)\left(\matr{A}^\T\matr{A}\hadam \matr{C}^\T\matr{C} + \lambda\matr{I}_{R\times R}\right)^{-1} 
 \end{equation*}}
 \STATE{\textbf{$\matr{S}$ estimate:} }
  \FOR{$i$ from $1$ to $R$}
 \STATE{$S_{j^*i} = 1 \iff j^* = \text{argmax}_j \frac{\langle \vect{B}_i | \vect{D}_j \rangle}{\| \vect{D}_j \|}$}
  \ENDFOR
\STATE{ \textbf{$\lambda$ update}
\IF{$\|\mB-\mD\mS \|_F^2 > 0.01\|\mB\|_F^2$} \STATE{$\lambda=p\lambda$}
\ENDIF
}
\STATE{\textbf{OUTPUT: } Estimated scores $\mS$  and factor
$\mB=\mD\mS$. }
 \end{algorithmic}
\caption{Smooth Matching Pursuit Alternating Least Squares}
\label{SMPALS}
\end{algorithm}

In the alternating outer loop, $\mB$ can be set either to the exact $\mD\mS$,
or to the approximate version computed by \eqref{flex-updt}. The first choice
provides the algorithm we called SMPALS, while the second is fully flexible,
therefore called Flex-MPALS.  For Flex-MPALS, $\lambda$ is also kept
constant. 
In SMPALS, until $\|\mB-\mD\mS \|_F^2 \geq 0.01\|\mB\|_F^2$ is reached, the coupling strength $\lambda$ increases by a multiplicative constant $p$.  
We chose a relatively aggressive choice for the increase $p=1.1$ for third-order tensors. For matrices, even a higher value of $p=1.5$ gave good results and allowed faster convergence. 

\paragraph*{Convergence of proposed algorithms}
 Although it is difficult to asses the convergence of MPALS, results for
SMPALS and Flex-MPALS can be derived. The convergence of SMPALS is guaranteed if $\lambda$ is
allowed to grow to infinity. Here is a sketch of the proof. 
The update of $B$ can be written as
follows; see~\eqref{flex-updt}: 
\begin{equation}
    \mB = \mD\mS + O\left( \frac{1}{\lambda} \right). 
\end{equation} 
This means that $\mB$ gets arbitrarily close to $\mD\mS$ as $\lambda$ goes to infinity. 
This implies that, for $\lambda$ sufficiently large, $\mS$ and $\mB$ will no
longer be modified: if $\mB$ does not change sufficiently, $\mS$ does not
change because of the discrete nature of the problem. 
In the mean time, the updates of $\mA$ and $\mC$, that use $\mD\mS$, decrease the
cost function and converge to a stationary point of the corresponding objective
function (for $\mS$ fixed)~\cite{GS00}. 

The convergence of Flex-MPALS, \textit{i.e.} with a fixed $\lambda$ and no
replacement $\mB$ by $\mD\mS$, is also guarantied for a normalized dictionary. Indeed, Flex-MPALS is a
block-coordinate descent algorithm where the blocks $\mA,\mB,\mC$ are updated using
an optimal least squares estimate, while the estimated $\mS$ in Flex-MPALS minimizes $\|\mB - \mD\mS
\|_F^2$ if the atoms of the dictionary have unit norm. Therefore the cost function decreases
at each step of Flex-MPALS hence convergence to some value (since it is bounded
below by zero).

\subsection{Continuous approaches}

As explained above, the original optimization problem underlying the DCPD model is combinatorial because of the $\ell_0$ pseudo-norm and the fact that $\mS$ is binary. 
To develop continuous algorithms such as gradient descent, 
a first step is to derive a continuous relaxation for the DCPD formulation. 
The first relaxations that comes to mind when working on $\ell_{0}/\ell_{p}$
pseudo-norm is a mixed norm $\ell_{1}/\ell_{2}$ which encourages sparsity column-wise. However in the DCPD model, $\mS$ can only have one non-zero coefficient per column, which makes a $\ell_{1}/\ell_{2}$ based optimization difficult to tune. 
Rather than mixed norm, we choose to use the $\ell_1$ norm as a sparsity enhancing penalization on all the entries of $\mS$, but under the constraint that the $\ell_2$ norm of the columns of $\mS$ are set to one (which is not convex). 
Indeed, if $\mathcal{B}(1,\ell_2)$ is the unit ball of the $\ell_2$ norm, the solutions to 
\begin{equation}
\underset{x\in{\mathcal{B}(1,\ell_2)}}{\text{argmin}} \|x\|_1
\end{equation}
are exactly located on the coordinate axes, which is also the constraint imposed on the columns of $\mS$. In light of this remark, solely imposing a $\ell_1$ penalization on coefficients of $\mS$, that could be easily done with for instance ADMM, may not be sufficient to ensure the high level of column sparsity that is sought. 

We chose to use the fast gradient from Nesterov~\cite{nesterov1983method} which can be understood as a proximal gradient descent with averaged steps. Although our problem is not convex, this method can still be applied and has been shown to work well in the non-convex case~\cite{ghadimi2016accelerated}. Fast gradient is used to solve the following subproblem in $\mS$  
\begin{equation}
\left\{
\begin{array}{l}
\min_{\mS} \gamma(\mS) = \frac{1}{2}\| \tT - \left(\mA\tensp\mD\mS\tensp\mC \right) \tens{I}_R \|_F^2 + \delta \|S\|_1 \\ 
\text{ such that } 
\|\vect{s}_i\|_2^2 = 1~~ \forall i\leq R, ~~ \mS \geq 0, 
\end{array}\right.
\label{costcon}
\end{equation} 
for a given parameter $\delta$ fixed by the user, and where  $\vect{s}_i$
 is the $i^{th}$ column of $\mS$. 
The non-negativity constraint on $\mS$ makes $\gamma$ differentiable with respect to $\mS$:
\begin{equation}
\frac{\partial \gamma}{\partial \mS} = \mD^T\mD\mS\left(\mA^T\mA\hadam\mC^T\mC\right)-\mD^T\matr{T}_2\left(\mA\kr\mC\right) + \delta \matr{1}_{d\times R}
\end{equation}
and the gradients of \eqref{costcon} with respect to $\mA$ and $\mC$ can be found for instance in \cite{royer2011computing} (we set $\mB=\mD\mS$) if an all-at-once optimization is sought.
The normalization constraint on the columns of $\mS$ is imposed by
normalization of $\mS$ at each iteration.  To avoid all values of a
column of $\mS$ to be non-positive, which would make the normalization
meaningless, the gradient step is constrained as follows to ensure all columns of $\mS$
after the gradient update have at least one positive entry:
\begin{equation}\label{safety}
    \text{step} = \min\left( \frac{1}{\epsilon_S},\min_j
\max_{i|g_{S}(i,j)>0} \frac{S(i,j)}{g_{S}(i,j)}-10^{-12} \right)
\end{equation}
where $\epsilon_S$ is the Lipschitz constant and $g_S$ is the gradient of the
cost function with respect to $\mS$. $M(i,j)$ refers to the entry of matrix $\mM$
indexed by $(i,j)$. 

The global optimization procedure we suggest is therefore a  mixture  of ALS and fast gradient, denoted ALS-FG and summarized in Algorithm \ref{ALS-FG}.
\begin{algorithm}
\begin{algorithmic}
\STATE \textbf{INPUT: } array $\tT$, factors $\mA$, $\mC$, initial scores $\mS$, dictionary $\mD$, regularization parameter $\delta > 0$, 
fast gradient parameter $\alpha \in ]0,1]$. 
 \STATE{~~\ul{step size computation}:\\
 $\epsilon_S =$ product of squared largest eigenvalues \\ of $ \mD^T\mD $ and $\left( \mA^T\mA\hadam\mC^T\mC \right)$}
  \WHILE{convergence criterion is not met}  
 \STATE{~~ \ul{gradient computation}:\\
 $\vect{g}_S=\mD^T\mD\mS\left(\mA^T\mA\hadam\mC^T\mC\right)-\mD^T\matr{T}_2\left(\mA\kr\mC\right) + \delta \matr{1}_{d\times R}$}
 \STATE{~~ \ul{gradient descent  correction:\\  
 }
 $\mS_{old}=\mS$ \\
  $\mS = \max(0,\mS - \text{step } \vect{g}_S) $ where step is computed as in
 \eqref{safety} to guarantee $\vect{s}_i \neq 0$ for all $i$  \\
  normalize columns of $\mS$ using the $\ell_2$ norm \\
  $\alpha_{old} = \alpha$ \\
 $\alpha= \frac{1}{2}(-\alpha_{old}^2+\sqrt{\alpha_{old}^4+4\alpha_{old})}$ \\
 $\beta = \frac{\alpha_{old}(1-\alpha_{old})}{(\alpha_{old}^2+\alpha)}$}
 \STATE{~~ \ul{update}:\\
 $\mS = \mS + \beta (\mS - \mS_{old}) $    \\
 $\mB=\mD\mS$}
\ENDWHILE 
 \STATE{\textbf{OUTPUT: } Estimated factor and scores $\mB$ and $\mS$.}

\end{algorithmic}
\caption{Fast Gradient for estimating $\mS$}
\label{ALS-FG}
\end{algorithm}
An important remark is that this algorithm is to be used inside an outer loop consisting of least squares problems. To ensure a smooth transition from the unconstrained problem to enforcing the dictionary, the penalty coefficient $\delta$ is linearly increased up from zero to a maximum value at the last outer iteration specified by the user. 
We observed that using the fast gradient improves convergence speed with respect to a simple coordinate descent, but other methods can also be implemented to solve the subproblem. 
Again, a flexible version of ALS-FG tackling optimization problem \eqref{flexlulu} can be derived from Algorithm \ref{ALS-FG} by adapting the gradient of $\gamma$ with respect to $\mS$ and using $\mB$ as a variable. 

\paragraph{Stopping criterion} Choosing the stopping criterion is application-dependent, and we let this parameter be tuned by interested users. Nevertheless, a baseline that we used in simulations below is to compute the residual error $E_i$ at iteration $i$ of the global optimization procedure, every few iterations, and to check whether the error is still decreasing enough by computing $\frac{|E_i - E_{i-1}|}{E_{i}}$. We set the number of iterations of the fast gradient inside ALS-FG to 10.

\paragraph{Normalization} An inherent ambiguity of the CPD model which is also present in the DCPD model is the scaling ambiguity, that is, the norm of columns of the factors is not determined solely by the model. It is possible to fix the norm of factors in both MPALS and ALS-FG by normalizing the columns of factor $\mA$ after each update of this factor, but this is not mandatory. 

\paragraph{Initialization} Because the global optimization problems underlying the DCPD model and its flexible counterparts are non-convex and the proposed optimization algorithms work locally, it is \textit{a priori} crucial to use a good initial point. 
If any information is available on the factors like non-negativity, it should be used in the initialization procedure. However, in the general case where no such information is available, one possible strategy is to compute an unconstrained CPD and use the obtained factors as an initialization for MPALS and ALS-FG. 
In the experiments below, we observed that indeed MPALS is highly sensitive to initialization. In particular, using a single random initialization leads in many cases to poor results for third-order tensors. Note however that is worked well for hyperspectral images. 



\section{Experiments on simulated and hyperspectral data}\label{exp-sec} 
\subsection{Simulated data experiments}
\subsubsection{Methodology} 

A critical question yet unanswered is whether using a dictionary within the
tensor factorization model improves on identification error,  which  is the
percentage of columns of $\mB$ correctly matched to atoms of $\mD$, with respect to
only projecting the result of an unconstrained factorization on the set of
atoms. Because information is added about factor $\mB$ in the model, the
identification error is expected to be smaller by using DCPD, in particular if
the simulations are made challenging in terms of signal to noise ratio (SNR),
rank under-estimation and over-estimation, or conditioning of the factors.

If the SNR is high,  if  the tensor is well-conditioned and  if  the rank is known, 
unconstrained CPD may provide relatively low identification error because of
the generic uniqueness of the solution.  For this reason, we choose difficult scenarios to compare DCPD to unconstrained CPD with outputs identified using a projection on the set of atoms. The dictionary used is constituted of atoms as follows: 
\begin{equation}
\begin{array}{ll}
\vect{d}^i_k = & | a_k \vect{u} + b_k + \nu_k^i \text{sinc}(\frac{\pi}{6}\vect{u}-e_k^i) \\
& + \mu_k^i \left(\text{tri}(\vect{u}-f_k^i+2) - \text{tri}(\vect{u}-f_k^i-2)\right) |
\end{array}
\end{equation} 
where $\vect{u}=\{1,\dots,L\}$, $a_k, b_k, \nu_k^i, \mu_k^i, e_k^i, f_k^i$ are
realizations of uniformly distributed random variables on respectively
$[-1/L,1/L], [-1,1], [-1/4, 1/4]$, $[-1/4, 1/4]$, $\{1,\dots,L\}$, $\{1,\dots,L\}$ and $\text{tri}$ is
a triangular pulse function, of support $[-2,2]$. Atoms are then normalized
with the $\ell_2$ norm. Index $k\in\{1,\dots,c\}$ is a class index, $c$ being
the number of classes, and $i\in\{1,\dots,\frac{d}{c}\}$ indexes individual atoms
in each class. This means that atoms of the dictionary are a sum of a
linear baseline common to groups of atoms and of two individual features; see Figure~\ref{fig1} for an
example.
This design is meant to bluntly echo spectral signatures of materials in
hyperspectral imaging, where only a small number of features discriminate materials from the same family.  For the experiments below, a single realization of the
dictionary was used, with no collinear columns (as to satisfy Proposition
\ref{prop2}), but it could not be checked whether its spark satisfied
the condition stated in Proposition \ref{prop1}. However, $\text{spark}(\mD)$
is equal to the dimension of the atoms plus one with probability one, which
ensures that the parameters of the DCPD models used in these simulations are
identifiable almost surely.   
\begin{figure}
\centering
\includegraphics[width=0.5\textwidth]{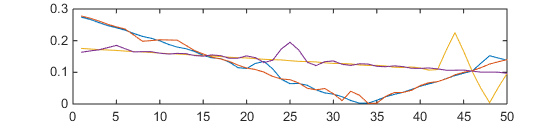}
\caption{Four dictionary atoms belonging to two different classes. \label{fig1}
\vspace*{-2em}} 
\end{figure}

The data are generated as follows: each entry of the factors $\mA$ and $\mC$ is drawn according
to a unitary centered Gaussian distribution, and the columns of $\mA$ and $\mC$ are then
normalized with the $\ell_2$ norm, which makes these factors well-conditioned
matrices. Matrix $\mS$ is  fixed in all experiments so that only one
atom of each class is contained in the columns of $\mB$.   White Gaussian
noise of fixed variance $\sigma$ is added to the data tensor.

Given the structure of the dictionary, the correlation among atoms of the same class can be quite high.  The dictionary is constituted of a large number of atoms ($d=1000$) subdivided in $c=50$ groups. 
The dimensions of the tensors are  $20 \times 50 \times 7$,  yielding an inter-group correlation of atoms that reaches $0.999$ at most. 
We choose a noise variance of $\sigma = 0.01$ which leads to an average SNR of
 about $11.5$dB,  and set the rank to $R=10$. 

 As explained above, the simulations focus on the impact on identification
performance when the rank is wrongly estimated, and when one of
the factor matrix, here $\mC$, is ill-conditioned. Therefore, we chose to grid
over an estimated rank $Re$ ranging from $7$ to $13$ with a good conditioning of
$\mC$. The impact of the conditionning of $\mC$ is studied through a grid on a
parameter $\rho$ such that given a randomly-drawn well conditioned $\mC^{(0)}$, 
it is modified as follows 
\begin{equation}\label{def-rho}
    \mC \leftarrow \mC^{(0)}\left(\rho\matr{I}_R + \frac{(1-\rho)}{R}\matr{1}_{R\times R}\right),
\end{equation}
and is then normalized. 
When $\rho$ is equal to 0, $\mC$ has column rank equal to 1, while its entries
follow i.i.d. Gaussian distributions when $\rho$ is 1. The estimated rank is set equal to the true rank when studying the
ill-conditioned case.  The number of realization in each setting is set to
$N=100$.

The CPD algorithm used as a baseline for comparison and for initialization is
the NWAY toolbox \cite{andersson2000n}. We compare projected results of the
NWAY toolbox with MPALS, SMPALS, Flex-MPALS and ALS-FG.  Coupling strength in
Flex-MPALS is fixed to $\lambda=0.04$ in this experiment. The value 0.04 provided spectra close to the
dictionary, while allowing some flexibility. In practice,  $\lambda$ could be
tuned for example in order to achieve a given relative error between the
columns of $B$ of the corresponding matched atoms of the dictionary (for
example 1\% relative error).  The number of iterations is fixed to a maximum of 1000, 
except for the rand-MPALS  (MPALS initialized randomly)  where the maximum iteration number is 5000.
Algorithm also stops if the stopping criterion mentioned in the previous
section goes below  $10^{-4}$. 

Finally, in both cases we use the NWAY toolbox for initialization, except for the randomly initialized MPALS which is initialized with factors following the same distribution as the true factors.

\subsubsection{Results}

 Figures \ref{fig-R} and \ref{fig-rho} report the mean identification
rate that measures the percentage of well-matched columns between $\mB$ and
$\widehat{\mB}$. In the case where the rank and the estimated rank are
miss-matched, the unmatched factors of $\mB$ and $\widehat{\mB}$ are considered
miss-matched. This leads to a best possible identification rate given by
$\frac{|Re-R|}{\max(Re,R))}$ and refered to as ``oracle''. Figure
\ref{fig-rho-err}
reports the relative mean square error on factor $\mB$ given by
$\mathds{E}\left[\frac{\|\mB -
\widehat{\mB}\Pi\|_F^2}{\|\mB\|_F^2}\right]$ where $\Pi$ is a
permutation matrix computed to match the estimated factors to the true factors.
Table \ref{table_t} reports the mean run time of
each algorithm in the first experiment when $R$ equals $Re$.  

\begin{figure}
\centering
\includegraphics[width=0.45\textwidth]{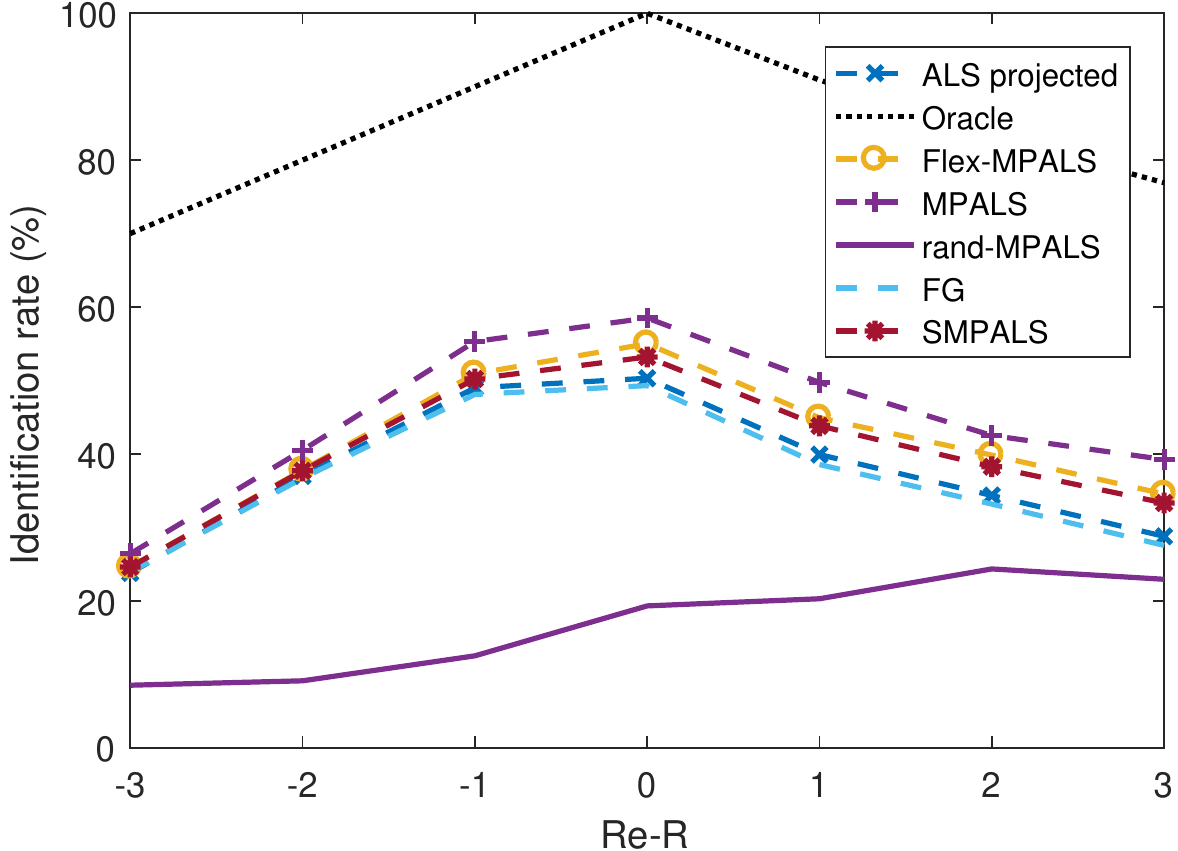}
\caption{Identification rates of various methods for varying estimated rank.
True rank $R$ is 10. The oracle is the best possible identification rate.}\label{fig-R}
\end{figure}

\begin{figure}
\centering
\includegraphics[width=0.45\textwidth]{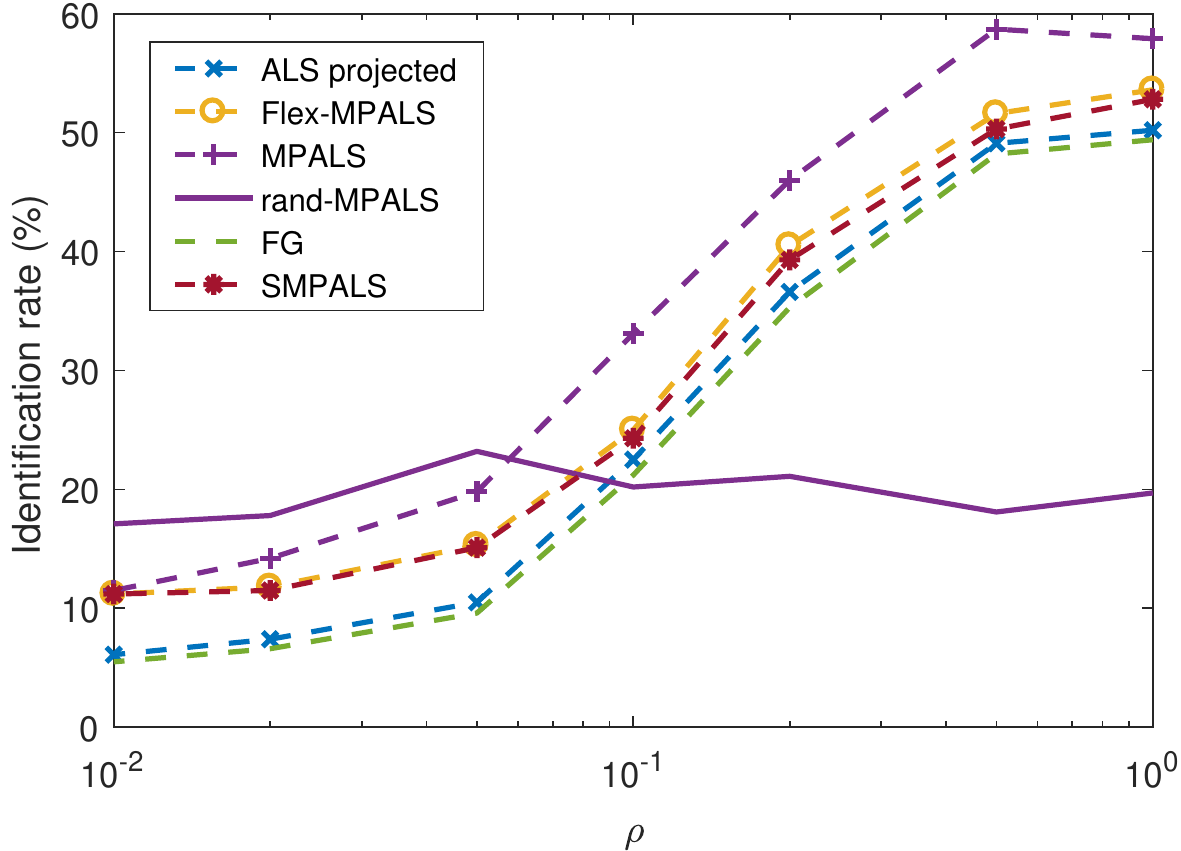}
\caption{Identification rates of various methods for various conditionning of
    factor matrix $\mC$, tuned using parameter $\rho$ defined in \eqref{def-rho}.
True rank $R$ is 10.\vspace*{-2em}}\label{fig-rho}
\end{figure}

\begin{figure}
\centering
\includegraphics[width=0.45\textwidth]{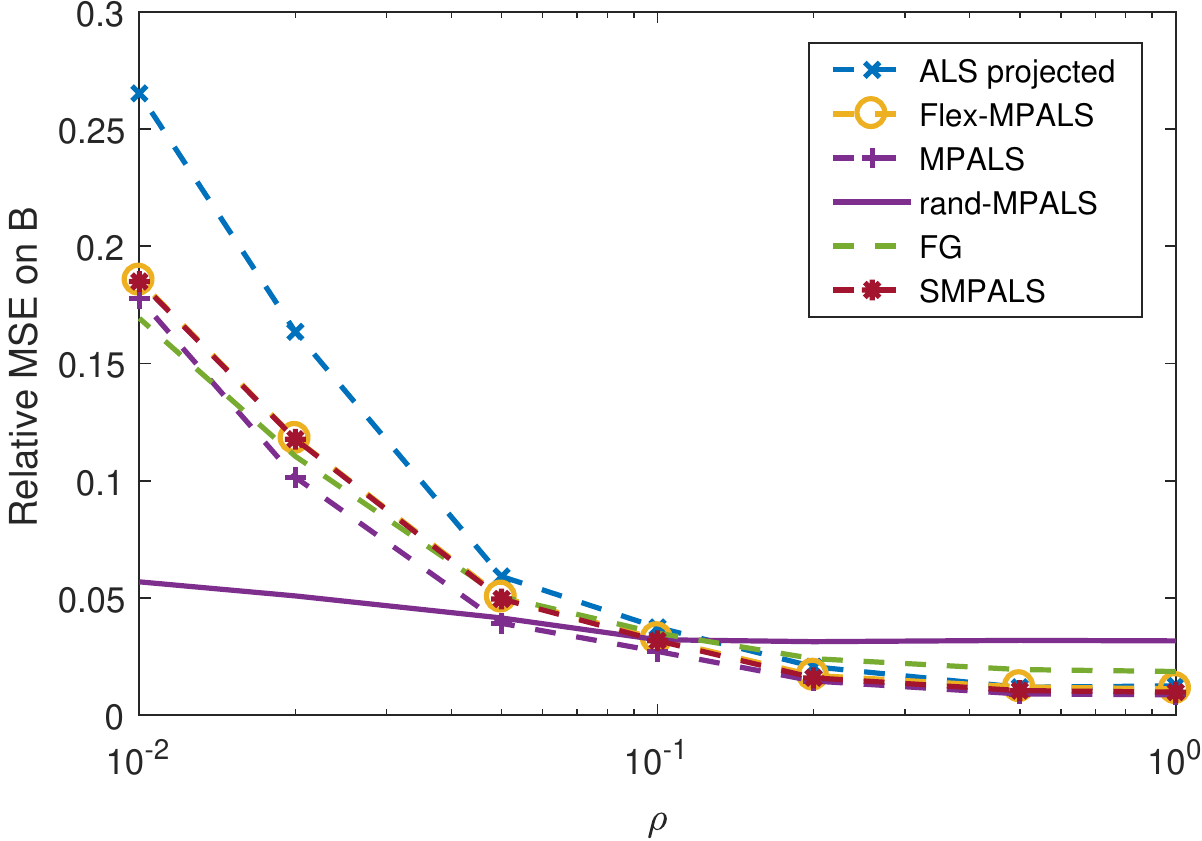}
\caption{Relative mean square error on factor $\mB$ that quantifies the
distance between the selected atoms and the true ones, as a function of $\rho$.}\label{fig-rho-err}
\end{figure}

\begin{table}
\centering 
\begin{tabular}{cccc}
    Alg. & ALS & MPALS & Flex-MPALS \\
    
    \hline
    Time (s) & 0.21 & 0.56 & 0.72 \\ 

    Alg.     & r-MPALS & SMPALS & FG \\
    \hline
    Time (s) & 2.75 & 0.73 & 7.62  
\end{tabular}
\caption{Average runtime of various algorithms over $N=100$ realization when
$R=Re$. \vspace*{-2em}}\label{table_t}
\end{table}

 As expected the identification rate increases with respect to projected
NWAY when using MPALS and its variants, with a gain ranging from a few percent
to over ten percent. This is a direct but non-obvious
consequence of the increased estimation performances. It can be seen that MPALS performs the best
in almost all cases, especially when the factor $\mC$ is ill-conditioned. However,
the wrong estimation of the rank does not impact the gain in performance between
projected CPD and DCPD. Notably, SMPALS performs slightly worse than
the naive MPALS. 

An important result to observe is the steady performance of the randomly
initialized MPALS over various values of $\rho$, both in terms of mean square
error and identification rate. Because DCPD is supposed to be
identifiable even  when  $\rho$ tends to $0$, the performance of MPALS should not
depend too much on it, as observed with the random initialization. However,
when the unconstrained CPD model is used as
an initialization method, the results of MPALS do depend heavily on
$\rho$. This shows that MPALS and all other proposed algorithm are very sensitive to initialization, a
fact also supported by the experiment in the next section. As a consequence, a
good initialization method is crucial for the performance of MPALS, this is a topic for further research.  

 Even though the identification rates evolve similarly for all methods in
the ill-conditioning experiment, the mean square error on $\mB$ tells a
different story. Since many atoms in the dictionary are very correlated, it is
reasonable to assume that the mean square error on $\mB$ could be small even
though all atoms are wrongly identified. It can be observed on Figure~\ref{fig-rho-err} that on average, all proposed algorithms for computing the DCPD
significantly outperform the projected ALS method when $\mC$ is
ill-conditioned. As a conclusion, the atoms picked by MPALS and its variants
are on average much closer to the true atoms than the ones picked using a
projected ALS method. This is especially true for the randomly initialized
MPALS. 

Finally, from this experiment, we observe that computing the DCPD with a
combinatorial greedy approach performs better than a continuous approach, since
the performance of ALS-FG is relatively poor. We believe the reason  for this
 is that the
greedy approaches can escape the basin of
attraction of a local minimum during the projection step, which a (standard) continuous approach will not
be able to do. In fact, we observed that, in most cases, once an entry is set
to zero by ALS-FG, it remains zero in the course of the iterations, which is
not the case for MPALS and its variants.

\subsection{Spectral Unmixing with the self-dictionary model}

For this communication on dictionary based tensor factorization models, we choose to try out the different models and algorithms on a well-known sparse coding problem, namely spectral unmixing under the pure-pixel assumption. 
As explained in Section \ref{sec-model}, 
the data itself can be used as the dictionary, and models described in this paper for matrices can be straightforwardly compared with state-of-the-art sparse coding approaches for spectral unmixing, namely the succesive projection algorithm (SPA)~\cite{MC01}, the successive nonnegative projection algorithm (SNPA)~\cite{G14b}, the Hierarchical Clustering algorithm (H2NMF)~\cite{GDK14} and FGNSR~\cite{Gillis2016fast}.  
Below, the spectra contained in the data are used as atoms, so that the dictionary has row dimension equals to the number of spectral bands ($\approx 150$), and number of atoms equals to the number of pixels in the hyperspectral images (HSI) ($\approx 10^5$). 

The two data sets that will be used to compare DCPD with these methods are
Urban and Terrain. 
These two HSI satisfy approximately the pure pixel assumption. The dimensions of the HSI after vectorization of the pixel dimensions are respectively $307^2  \times 162 $ and $(500\times 307) \times 166 $. We chose to decompose the two data sets with respective ranks $R_{urban}=6$ and $R_{terrain}=5$ according to what has been done previously in the literature; see~\cite{Gillis2016fast} and the references therein. 
The maximal number of iterations of the DCPD algorithms is set to 50. 
Also because it does not perform well, results for the continuous fast gradient algorithm are not presented below.

We initialize the MPALS algorithm and its variants with each state-of-the-art
algorithm with 10 additional A-HALS steps. For instance MPALS and SMPALS
initialized with SPA are respectively denoted as d-SPA and ds-SPA. Results are
presented in Table \ref{tab5}.  Additionally, considering the performance of
projected ALS in the multiway experiment above, the performance of projected
NMF is also shown, for instance nmf-SPA refers to the A-HALS algorithm
initialized with SPA, which endmembers output are projected onto the data
points and which abundances are  re-estimated  using non-negative least squares.  

\begin{table}[h!]\small
    \begin{center}
\begin{tabular}{|c|c|c|c|c|}
\hline
 &  \multicolumn{2}{c|}{Urban HSI}  & \multicolumn{2}{c|}{Terrain HSI} \\
 & time(s) & Rel.err. & time(s) & Rel.err. \\ 
\hline SPA & 0.2 & 9.58 &  0.3 & 5.89 \\
	   d-SPA & 20 & 4.57 & 48 & 3.50 \\
	   ds-SPA & 21 & 4.67 &  58 & 3.37 \\
       nmf-SPA & 13 & 5.37 & 18 & 3.75 \\
\hline VCA & 2.0 & 13.07 &  1.6 & 18.61 \\
	   d-VCA & 16 & 4.73 & 46 & 3.29 \\
	   ds-VCA & 20 & 4.66 &  35 & 3.37 \\
       nmf-VCA & 12 & 6.63 & 17  & 4.27 \\
\hline SNPA & 9.5 & 9.63 & 13 & 5.76\\
	   d-SNPA & 17 & 5.02 & 47 & 3.78\\
   	   ds-SNPA & 21 & 4.91 & 59 & 3.88\\
       nmf-SNPA & 13 & 5.14 & 19  & 4.47 \\
\hline H2NMF & 8 & 5.81 & 11.7 & 5.09\\
	   d-H2NMF & 17 & \textbf{4.05} & 47 & 3.49\\
   	   ds-H2NMF & 21 & \textbf{4.05} & 64 & 3.49\\
       nmf-H2NMF & 11 & 5.03 & 17 & 4.51 \\
\hline FGNRS-100 & 2.2 & 5.58 & 1.6 & 3.34\\
	   d-FGNSR-100 & 17 & 4.47 & 45 & \textbf{3.01} \\
	   ds-FGNSR-100 & 20 & 4.47 & 58 & \textbf{3.01} \\
       nmf-FGNSR-100 & 12 & 6.37 & 19 & 3.31 \\
\hline
\end{tabular} 
\vspace{1em}
\caption{Reconstruction error $(\%)$ for the Urban and Terrain airport HSI. Best results are highlighted in bold. 
Computation time does not include 500 iterations of nonnegative least squares
update for the abundances after each method (which takes about 10 seconds on
average).\vspace*{-2em}}
\label{tab5}
\end{center}  
\end{table}

The relative reconstruction error is used as a performance metric, since it is
not possible to assess the identification performance without a ground-truth.
Therefore, this experiment only studies the efficiency of the MPALS algorithm
for minimizing the objective function. 

In all cases, using either SMPALS or MPALS improves on the initial values for identified spectra. As shown already in the previous experiment, initialization plays an important role in the final reconstruction error, but at least using the DCPD model always refines the solutions, even when it is initially low like for FGNSR-100. 
Moreover, the DCPD algorithms are not excessively costly with respect to other state-of-the-art methods. There is no significant difference between MPALS and its smooth counterpart. Figure \ref{fig-abund} shows the estimated abundances and spectra with the H2NMF-MPALS algorithm for both Urban and Terrain, and materials can be identified for each component by the user (no ground truth is available). Finally, Figure  \ref{fig-spectra} shows the reconstruction error map on Urban and Terrain HSIs for the H2NMF-MPALS. 
Clearly the remaining error is not distributed as an i.i.d. Gaussian noise, which means the Frobenius norm used as the distance metric is not adapted. Also, most of the remaining error comes from rooftops and roads in the Urban HSI, a zone probably corrupted by large spectral variability. 

Because Flex-MPALS does not require exactly that the pure pixel assumption is verified, we do not include it in this simulation. Indeed, using a reconstruction error criteria, Flex-MPALS would outperform the other methods, but that would not mean that obtained abundances and endmembers are better. On the other hand, the flexible model better tackles spectral variability since it can modify to some extent the spectra extracted from the pure pixels.\footnote{Interested readers will find the code for matrix Flex-MPALS online.}

\begin{figure*}
\centering
\includegraphics[width=\textwidth]{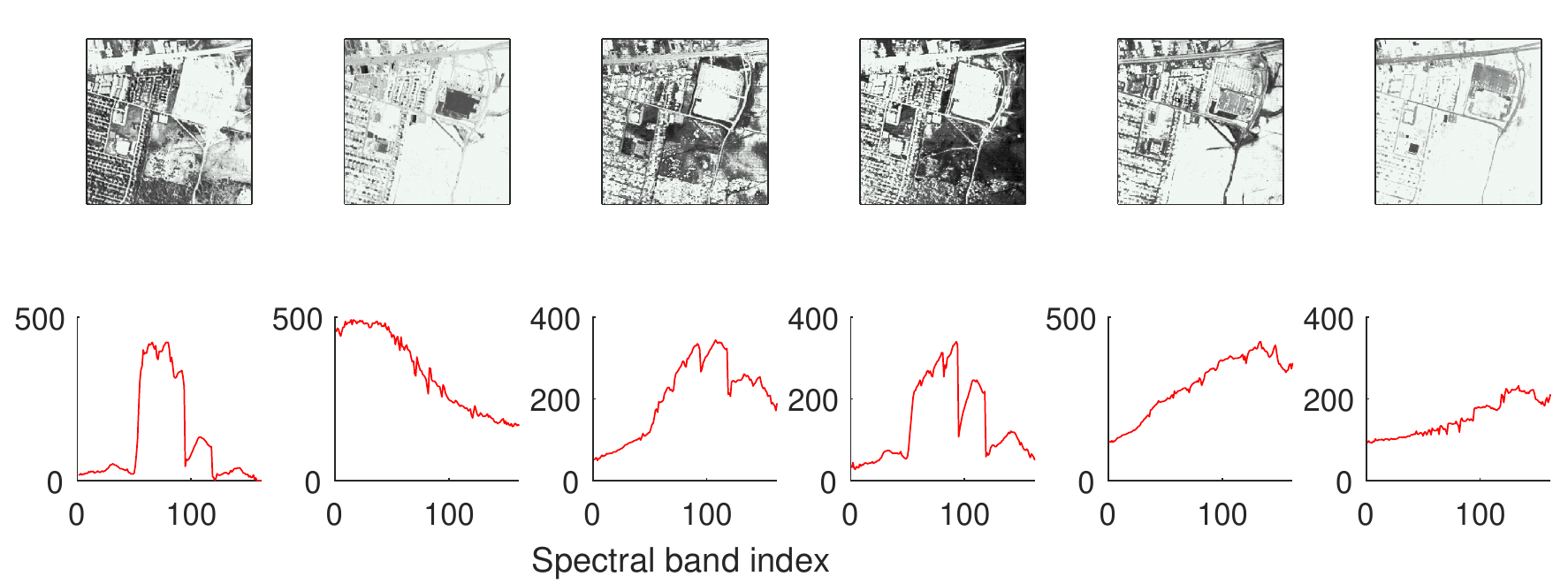}\\
\includegraphics[width=\textwidth]{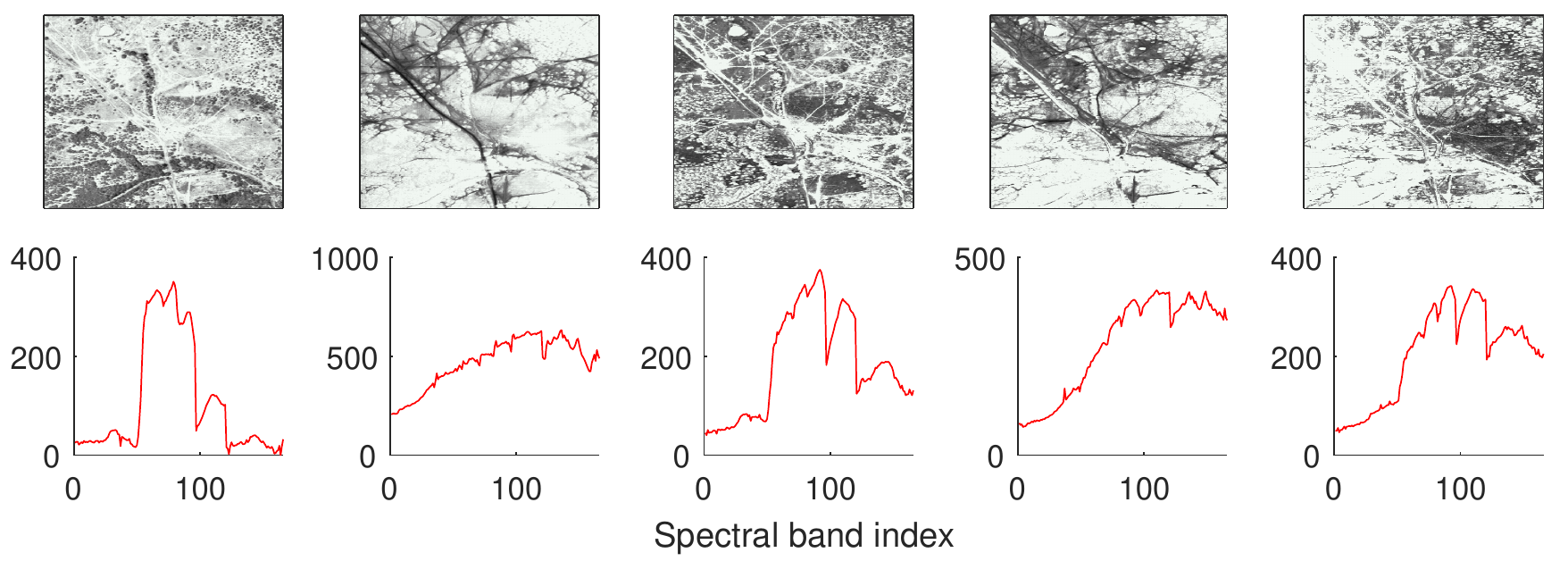}
\caption{Estimated spectra on the Urban HSI with d-H2NMF (top) and on the Terrain HSI with d-H2NMF (bottom).}
\label{fig-abund}
\end{figure*}

\begin{figure*}
\centering
\includegraphics[width=0.4\textwidth]{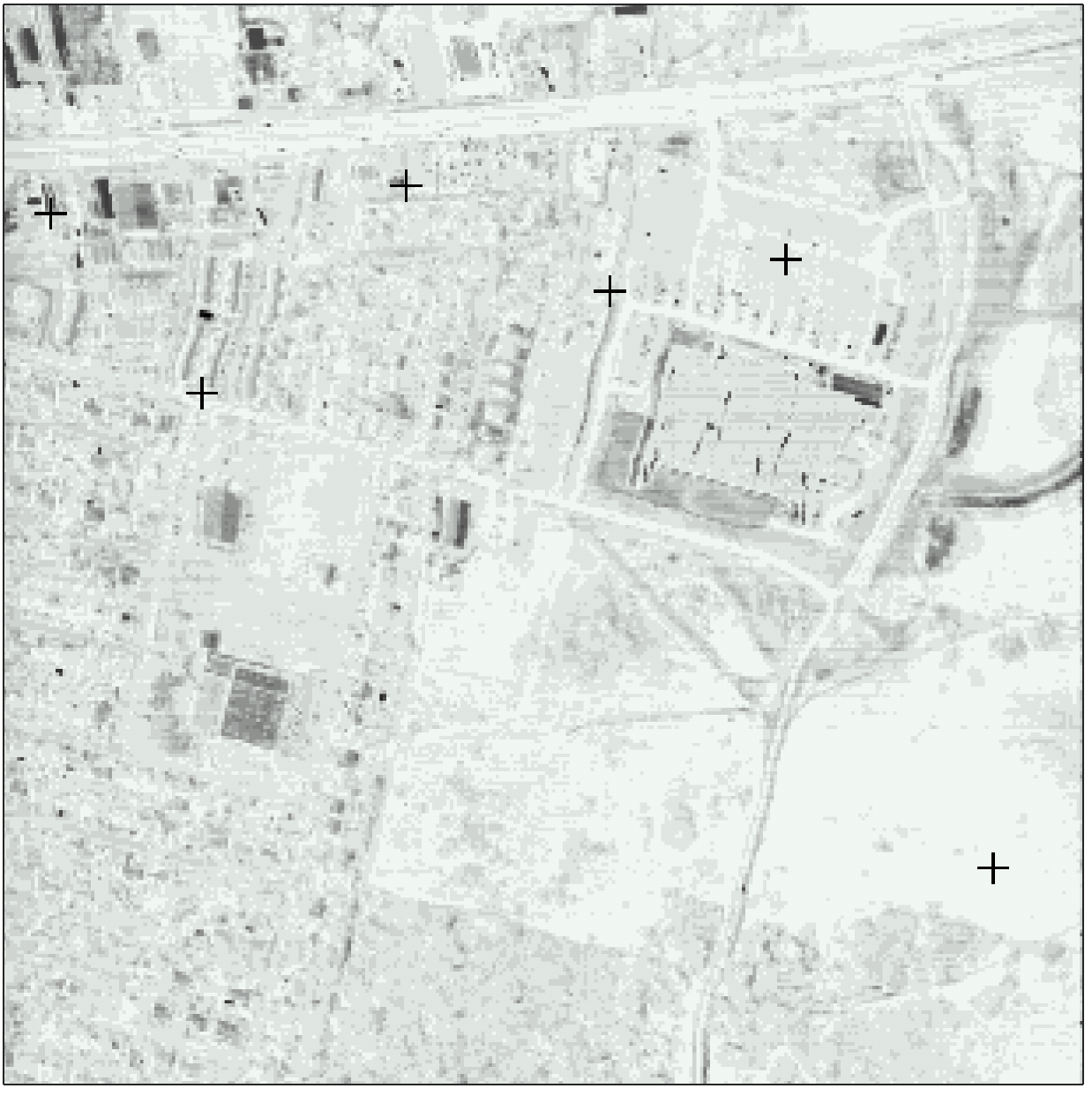}
\includegraphics[width=0.419\textwidth]{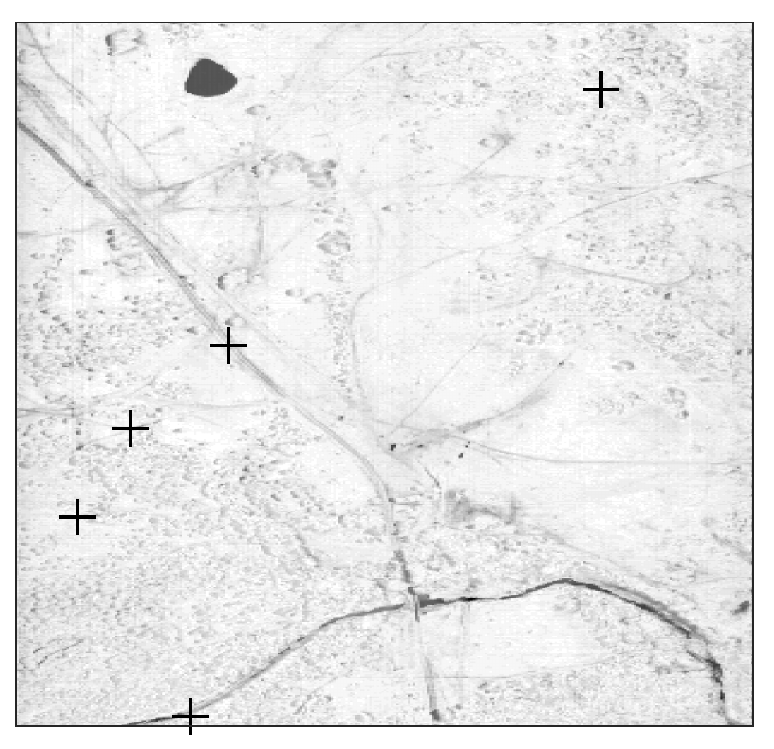}\\
\caption{Residual error maps on the Urban (left) and Terrain (right) HSIs for d-H2NMF. The black crosses mark the selected pure pixels.}
\label{fig-spectra}
\end{figure*}



\section{Conclusions}

To  jointly separate and identify sources using tensor canonical polyadic
decomposition and a known dictionary, we introduced in this paper the DCPD model along with some flexible variants. Identifiability of the DCPD model parameters was discussed in both the matrix and the higher-order tensor cases. 
We proposed both greedy and continuous algorithms for DCPD and compared them on synthetic data sets and hyperspectral images. 
We observed that 
(i) the greedy algorithms provide in most cases better results, and that 
(ii) our DCPD formulation improves results with respect to 
standard CPD on the synthetic data sets, 
and  with respect to spectral unmixing approaches based either on convex
relaxation or geometric methods. 
 A particularly promising direction for further research would be the 
design of efficient initialization schemes for the proposed greedy algorithms
which are particularly sensitive to initialization. 


An interesting continuation of this work would  also  be to compute the Cram\'er Rao bounds of the parameters in the DCPD to formally support what was presented in the simulation section. Moreover, in most application cases the dictionary is either unknown, or describes only a subset of the components. Thus a partial DCPD model should be investigated, as well as a dictionary learning scheme for tensors following the DCPD model.


\section*{Acknowledgment}
The authors would like to thank Pierre Comon, Rodrigo Cabral Farias and Miguel A.\@ Veganzones for early fruitful discussions on this work.


\ifCLASSOPTIONcaptionsoff
  \newpage
\fi

\bibliographystyle{IEEEtran}

\begin{thebibliography}{10}
\providecommand{\url}[1]{#1}
\csname url@samestyle\endcsname
\providecommand{\newblock}{\relax}
\providecommand{\bibinfo}[2]{#2}
\providecommand{\BIBentrySTDinterwordspacing}{\spaceskip=0pt\relax}
\providecommand{\BIBentryALTinterwordstretchfactor}{4}
\providecommand{\BIBentryALTinterwordspacing}{\spaceskip=\fontdimen2\font plus
\BIBentryALTinterwordstretchfactor\fontdimen3\font minus
  \fontdimen4\font\relax}
\providecommand{\BIBforeignlanguage}[2]{{%
\expandafter\ifx\csname l@#1\endcsname\relax
\typeout{** WARNING: IEEEtran.bst: No hyphenation pattern has been}%
\typeout{** loaded for the language `#1'. Using the pattern for}%
\typeout{** the default language instead.}%
\else
\language=\csname l@#1\endcsname
\fi
#2}}
\providecommand{\BIBdecl}{\relax}
\BIBdecl

\bibitem{comon2010handbook}
P.~Comon and C.~Jutten, \emph{Handbook of Blind Source Separation: Independent
  component analysis and applications}.\hskip 1em plus 0.5em minus 0.4em\relax
  Academic press, 2010.

\bibitem{SidiBG00:ieeesp}
N.~D. Sidiropoulos, R.~Bro, and G.~B. Giannakis, ``Parallel factor analysis in
  sensor array processing,'' \emph{IEEE Trans. Sig. Proc.}, vol.~48, no.~8, pp.
  2377--2388, Aug. 2000.

\bibitem{Bro1998}
R.~Bro, ``\textit{Multi-way Analysis in the Food Industry: Models, Algorithms,
  and Applications},'' Ph.D. dissertation, University of Amsterdam, The
  Netherlands, 1998.

\bibitem{veganzones2016nonnegative}
M.~A. Veganzones, J.~E. Cohen, R.~Cabral~Farias, J.~Chanussot, and P.~Comon,
  ``Nonnegative tensor {CP} decomposition of hyperspectral data,''
  \emph{Geoscience and Remote Sensing, IEEE Transactions on}, vol.~52, pp.
  2577--2588, 2016.

\bibitem{cichocki2015tensor}
A.~Cichocki, D.~Mandic, L.~De~Lathauwer, G.~Zhou, Q.~Zhao, C.~Caiafa, and H.~A.
  Phan, ``Tensor decompositions for signal processing applications: From
  two-way to multiway component analysis,'' \emph{Signal Processing Magazine,
  IEEE}, vol.~32, no.~2, pp. 145--163, 2015.

\bibitem{kroonenberg1983three}
P.~M. Kroonenberg, \emph{Three-mode principal component analysis: Theory and
  applications}.\hskip 1em plus 0.5em minus 0.4em\relax DSWO press, 1983,
  vol.~2.

\bibitem{papalexakis2016tensors}
E.~E. Papalexakis, C.~Faloutsos, and N.~D. Sidiropoulos, ``Tensors for data
  mining and data fusion: Models, applications, and scalable algorithms,''
  \emph{ACM Transactions on Intelligent Systems and Technology (TIST)}, vol.~8,
  no.~2, p.~16, 2016.

\bibitem{olshausen1997sparse}
B.~A. Olshausen and D.~J. Field, ``Sparse coding with an overcomplete basis
  set: A strategy employed by {V1}?'' \emph{Vision research}, vol.~37, no.~23,
  pp. 3311--3325, 1997.

\bibitem{elhamifar2012see}
E.~Elhamifar, G.~Sapiro, and R.~Vidal, ``See all by looking at a few: Sparse
  modeling for finding representative objects,'' in \emph{Computer Vision and
  Pattern Recognition (CVPR), 2012 IEEE Conference on}.\hskip 1em plus 0.5em
  minus 0.4em\relax IEEE, 2012, pp. 1600--1607.

\bibitem{G14b}
N.~Gillis, ``Successive nonnegative projection algorithm for robust nonnegative
  blind source separation,'' \emph{SIAM J. Imaging Sci.}, vol.~7, no.~2, pp.
  1420--1450, 2014.

\bibitem{ND05}
J.~Nascimento and J.~Dias, ``Vertex component analysis: a fast algorithm to
  unmix hyperspectral data,'' \emph{{IEEE} Transactions on Geoscience and
  Remote Sensing}, vol.~43, no.~4, pp. 898--910, 2005.

\bibitem{KSK12}
A.~Kumar, V.~Sindhwani, and P.~Kambadur, ``Fast conical hull algorithms for
  near-separable non-negative matrix factorization,'' in \emph{Int. Conf. on
  Machine Learning (ICML '13)}, 2013, vol.~28, no.~1, pp. 231--239.

\bibitem{GV14}
N.~Gillis and S.~A. Vavasis, ``Fast and robust recursive algorithmsfor
  separable nonnegative matrix factorization,'' \emph{{IEEE} Transactions on
  Pattern Analysis and Machine Intelligence}, vol.~36, no.~4, pp. 698--714,
  2014.

\bibitem{Hack12}
W.~Hackbusch, \emph{Tensor Spaces and Numerical Tensor Calculus}, ser. Series
  in Computational Mathematics.\hskip 1em plus 0.5em minus 0.4em\relax Berlin,
  Heidelberg: Springer, 2012.

\bibitem{cohen2015notations}
J.~E. Cohen, ``About notations in multiway array processing,'' \emph{arXiv
  preprint arXiv:1511.01306}, 2015.

\bibitem{kruskal1977three}
J.~B. Kruskal, ``Three-way arrays: rank and uniqueness of trilinear
  decompositions, with application to arithmetic complexity and statistics,''
  \emph{Linear algebra and its applications}, vol.~18, no.~2, pp. 95--138,
  1977.

\bibitem{domanov2013uniqueness}
I.~Domanov and L.~De~Lathauwer, ``On the uniqueness of the canonical polyadic
  decomposition of third-order tensors---{Part II}: Uniqueness of the overall
  decomposition,'' \emph{SIAM Journal on Matrix Analysis and Applications},
  vol.~34, no.~3, pp. 876--903, 2013.

\bibitem{landsberg2013equations}
J.~M. Landsberg and G.~Ottaviani, ``Equations for secant varieties of veronese
  and other varieties,'' \emph{Annali di Matematica Pura ed Applicata}, vol.
  192, no.~4, pp. 569--606, 2013.

\bibitem{Rao65}
C.~R. Rao, \emph{Linear Statistical Inference and its Applications}, ser.
  Probability and Statistics.\hskip 1em plus 0.5em minus 0.4em\relax Wiley,
  1965.

\bibitem{bro2003new}
R.~Bro and H.~A. Kiers, ``A new efficient method for determining the number of
  components in {PARAFAC} models,'' \emph{Journal of chemometrics}, vol.~17,
  no.~5, pp. 274--286, 2003.

\bibitem{da2008robust}
J.~C. da~Costa, M.~Haardt, and F.~Romer, ``Robust methods based on the hosvd
  for estimating the model order in {PARAFAC} models,'' in \emph{Sensor Array
  and Multichannel Signal Processing Workshop, 2008. SAM 2008. 5th IEEE}.\hskip
  1em plus 0.5em minus 0.4em\relax IEEE, 2008, pp. 510--514.

\bibitem{valeur2012molecular}
B.~Valeur and M.~N. Berberan-Santos, \emph{Molecular fluorescence: principles
  and applications}.\hskip 1em plus 0.5em minus 0.4em\relax John Wiley \& Sons,
  2012.

\bibitem{lee1999learning}
D.~D. Lee and H.~S. Seung, ``Learning the parts of objects by non-negative
  matrix factorization,'' \emph{Nature}, vol. 401, no. 6755, pp. 788--791,
  1999.

\bibitem{Jose12}
J.~Bioucas-Dias, A.~Plaza, N.~Dobigeon, M.~Parente, Q.~Du, P.~Gader, and
  J.~Chanussot, ``Hyperspectral unmixing overview: {G}eometrical, statistical,
  and sparse regression-based approaches,'' \emph{IEEE Journal of Selected
  Topics in Applied Earth Observations and Remote Sensing}, vol.~5, no.~2, pp.
  354--379, 2012.

\bibitem{zare2014endmember}
A.~Zare and K.~Ho, ``Endmember variability in hyperspectral analysis:
  Addressing spectral variability during spectral unmixing,'' \emph{IEEE Signal
  Processing Magazine}, vol.~31, no.~1, pp. 95--104, 2014.

\bibitem{CabrCC16:tsp}
R.~Cabral-Farias, J.~E. Cohen, and P.~Comon, ``Exploring multimodal data fusion
  through joint decompositions with flexible couplings,'' \emph{IEEE Trans.
  Sig. Proc.}, vol.~64, no.~18, pp. 4830--4844, 2016.

\bibitem{guo2012comparison}
Y.~Guo and M.~Berman, ``A comparison between subset selection and l1
  regularisation with an application in spectroscopy,'' \emph{Chemometrics and
  Intelligent Laboratory Systems}, vol. 118, pp. 127--138, 2012.

\bibitem{harshman1972parafac2}
R.~A. Harshman, ``{PARAFAC2}: Mathematical and technical notes,'' \emph{UCLA
  working papers in phonetics}, vol.~22, no. 3044, p. 122215, 1972.

\bibitem{kiers1999parafac2}
H.~A. Kiers, J.~M. Ten~Berge, and R.~Bro, ``{PARAFAC2-Part I}. {A} direct
  fitting algorithm for the {PARAFAC2} model,'' \emph{Journal of Chemometrics},
  vol.~13, no. 3-4, pp. 275--294, 1999.

\bibitem{DS03}
D.~Donoho and V.~Stodden, ``{When does non-negative matrix factorization give a
  correct decomposition into parts?}'' in \emph{In Advances in Neural
  Information Processing 16}, 2003.

\bibitem{arora2012computing}
S.~Arora, R.~Ge, R.~Kannan, and A.~Moitra, ``Computing a nonnegative matrix
  factorization--provably,'' in \emph{Proceedings of the forty-fourth annual
  ACM symposium on Theory of computing}.\hskip 1em plus 0.5em minus 0.4em\relax
  ACM, 2012, pp. 145--162.

\bibitem{ang2016non}
A.~M.~S. Ang, Y.~S. Hung, and Z.~Zhang, ``A non-negative tensor factorization
  approach to feature extraction for image analysis,'' in \emph{Digital Signal
  Processing (DSP), 2016 IEEE International Conference on}.\hskip 1em plus
  0.5em minus 0.4em\relax IEEE, 2016, pp. 168--178.

\bibitem{iordache2014collaborative}
M.-D. Iordache, J.~Bioucas-Dias, and A.~Plaza, ``Collaborative sparse
  regression for hyperspectral unmixing,'' \emph{IEEE Transactions on
  Geoscience and Remote Sensing}, vol.~52, no.~1, pp. 341--354, 2014.

\bibitem{EMO12}
E.~Esser, M.~Moller, S.~Osher, G.~Sapiro, and J.~Xin, ``A convex model for
  nonnegative matrix factorization and dimensionality reduction on physical
  space,'' \emph{{IEEE} Transactions on Image Processing}, vol.~21, no.~7, pp.
  3239--3252, 2012.

\bibitem{BRRT12}
V.~Bittorf, B.~Recht, E.~R\'{e}, and J.~Tropp, ``{Factoring nonnegative
  matrices with linear programs},'' in \emph{Advances in Neural Information
  Processing Systems (NIPS~'12)}, 2012, pp. 1223--1231.

\bibitem{GL13}
N.~Gillis and R.~Luce, ``Robust near-separable nonnegative matrix factorization
  using linear optimization,'' \emph{J. Mach. Learn. Res.}, vol.~15, pp.
  1249--1280, 2014.

\bibitem{sahnoun2017simultaneous}
S.~Sahnoun, E.-H. Djermoune, D.~Brie, and P.~Comon, ``A simultaneous sparse
  approximation method for multidimensional harmonic retrieval,'' \emph{Signal
  Processing}, vol. 131, pp. 36--48, 2017.

\bibitem{MC01}
U.~Ara\'ujo, B.~Saldanha, R.~Galv\~ao, T.~Yoneyama, H.~Chame, and V.~Visani,
  ``The successive projections algorithm for variable selection in
  spectroscopic multicomponent analysis,'' \emph{Chemometr. Intell. Lab.},
  vol.~57, no.~2, pp. 65--73, 2001.

\bibitem{RC03}
H.~Ren and C.-I. Chang, ``Automatic spectral target recognition in
  hyperspectral imagery,'' \emph{IEEE Trans. on Aerospace and Electronic
  Systems}, vol.~39, no.~4, pp. 1232--1249, 2003.

\bibitem{CM11}
T.-H. Chan, W.-K. Ma, A.~Ambikapathi, and C.-Y. Chi, ``A simplex volume
  maximization framework for hyperspectral endmember extraction,'' \emph{IEEE
  Trans. on Geoscience and Remote Sensing}, vol.~49, no.~11, pp. 4177--4193,
  2011.

\bibitem{Gillis2016fast}
N.~{Gillis} and R.~{Luce}, ``{A Fast Gradient Method for Nonnegative Sparse
  Regression with Self Dictionary},'' \emph{IEEE Transactions on Image
  Processing}, 2017, to appear, doi: 10.1109/TIP.2017.2753400.

\bibitem{de2000multilinear}
L.~De~Lathauwer, B.~De~Moor, and J.~Vandewalle, ``A multilinear singular value
  decomposition,'' \emph{SIAM journal on Matrix Analysis and Applications},
  vol.~21, no.~4, pp. 1253--1278, 2000.

\bibitem{timmerman2002three}
M.~E. Timmerman and H.~A. Kiers, ``Three-way component analysis with smoothness
  constraints,'' \emph{Computational statistics \& data analysis}, vol.~40,
  no.~3, pp. 447--470, 2002.

\bibitem{favier2014overview}
G.~Favier and A.~de~Almeida, ``Overview of constrained {PARAFAC} models,''
  \emph{EURASIP Journal on Advances in Signal Processing}, vol. 2014, no.~1, p.
  142, 2014.

\bibitem{goulart2016tensor}
J.~H.~M. Goulart, M.~Boizard, R.~Boyer, G.~Favier, and P.~Comon, ``Tensor cp
  decomposition with structured factor matrices: Algorithms and performance,''
  \emph{IEEE Journal of Selected Topics in Signal Processing}, vol.~10, no.~4,
  pp. 757--769, 2016.

\bibitem{cohen2017new}
J.~E. Cohen and N.~Gillis, ``A new approach to dictionary-based nonnegative
  matrix factorization,'' in \emph{Signal Processing Conference (EUSIPCO), 2017
  Proceedings of the 25th European}.\hskip 1em plus 0.5em minus 0.4em\relax
  IEEE, 2017, to appear.

\bibitem{domanov2016generic}
I.~Domanov and L.~De~Lathauwer, ``Generic uniqueness of a structured matrix
  factorization and applications in blind source separation,'' \emph{IEEE
  Journal of Selected Topics in Signal Processing}, vol.~10, no.~4, pp.
  701--711, 2016.

\bibitem{roberts1998mapping}
D.~A. Roberts, M.~Gardner, R.~Church, S.~Ustin, G.~Scheer, and R.~Green,
  ``Mapping chaparral in the santa monica mountains using multiple endmember
  spectral mixture models,'' \emph{Remote Sensing of Environment}, vol.~65,
  no.~3, pp. 267--279, 1998.

\bibitem{combe2008analysis}
J.-P. Combe, S.~Le~Mouelic, C.~Sotin, A.~Gendrin, J.~Mustard, L.~Le~Deit,
  P.~Launeau, J.-P. Bibring, B.~Gondet, Y.~Langevin \emph{et~al.}, ``Analysis
  of omega/mars express data hyperspectral data using a multiple-endmember
  linear spectral unmixing model (melsum): Methodology and first results,''
  \emph{Planetary and Space Science}, vol.~56, no.~7, pp. 951--975, 2008.

\bibitem{song2005spectral}
C.~Song, ``Spectral mixture analysis for subpixel vegetation fractions in the
  urban environment: How to incorporate endmember variability?'' \emph{Remote
  Sensing of Environment}, vol.~95, no.~2, pp. 248--263, 2005.

\bibitem{asner2003scale}
G.~P. Asner, M.~M. Bustamante, and A.~R. Townsend, ``Scale dependence of
  biophysical structure in deforested areas bordering the tapajos national
  forest, central amazon,'' \emph{Remote Sensing of Environment}, vol.~87,
  no.~4, pp. 507--520, 2003.

\bibitem{domanov2013study}
I.~Domanov, ``Study of canonical polyadic decomposition of higher-order
  tensors,'' Ph.D. dissertation, KU Leuven, 2013.

\bibitem{de2008tensor}
V.~De~Silva and L.-H. Lim, ``Tensor rank and the ill-posedness of the best
  low-rank approximation problem,'' \emph{SIAM Journal on Matrix Analysis and
  Applications}, vol.~30, no.~3, pp. 1084--1127, 2008.

\bibitem{tropp2004greed}
J.~A. Tropp, ``Greed is good: Algorithmic results for sparse approximation,''
  \emph{IEEE Transactions on Information theory}, vol.~50, no.~10, pp.
  2231--2242, 2004.

\bibitem{tropp2006just}
------, ``Just relax: Convex programming methods for identifying sparse signals
  in noise,'' \emph{IEEE transactions on information theory}, vol.~52, no.~3,
  pp. 1030--1051, 2006.

\bibitem{bro1997fast}
R.~Bro and S.~De~Jong, ``A fast non-negativity-constrained least squares
  algorithm,'' \emph{Journal of chemometrics}, vol.~11, no.~5, pp. 393--401,
  1997.

\bibitem{gillis2012accelerated}
N.~Gillis and F.~Glineur, ``Accelerated multiplicative updates and hierarchical
  als algorithms for nonnegative matrix factorization,'' \emph{Neural
  computation}, vol.~24, no.~4, pp. 1085--1105, 2012.

\bibitem{cichocki2002adaptive}
A.~Cichocki and S.-i. Amari, \emph{Adaptive blind signal and image processing:
  learning algorithms and applications}.\hskip 1em plus 0.5em minus 0.4em\relax
  John Wiley \& Sons, 2002, vol.~1.

\bibitem{kuhn1955hungarian}
H.~W. Kuhn, ``The hungarian method for the assignment problem,'' \emph{Naval
  Research Logistics (NRL)}, vol.~2, no. 1-2, pp. 83--97, 1955.

\bibitem{GS00}
L.~Grippo and M.~Sciandrone, ``On the convergence of the block nonlinear
  {G}auss {S}eidel method under convex constraints,'' \emph{Operations Research
  Letters}, vol.~26, pp. 127--136, 2000.

\bibitem{nesterov1983method}
Y.~Nesterov, ``A method of solving a convex programming problem with
  convergence rate o (1/k2),'' in \emph{Soviet Mathematics Doklady}, vol.~27,
  no.~2, 1983, pp. 372--376.

\bibitem{ghadimi2016accelerated}
S.~Ghadimi and G.~Lan, ``Accelerated gradient methods for nonconvex nonlinear
  and stochastic programming,'' \emph{Mathematical Programming}, vol. 156, no.
  1-2, pp. 59--99, 2016.

\bibitem{royer2011computing}
J.-P. Royer, N.~Thirion-Moreau, and P.~Comon, ``Computing the polyadic
  decomposition of nonnegative third order tensors,'' \emph{Signal Processing},
  vol.~91, no.~9, pp. 2159--2171, 2011.

\bibitem{andersson2000n}
C.~A. Andersson and R.~Bro, ``The {N}-way toolbox for matlab,''
  \emph{Chemometrics and intelligent laboratory systems}, vol.~52, no.~1, pp.
  1--4, 2000.

\bibitem{GDK14}
N.~Gillis, D.~Kuang, and H.~Park, ``Hierarchical clustering of hyperspectral
  images using rank-two nonnegative matrix factorization,'' \emph{{IEEE} Trans.
  Geosci. Remote Sensing}, vol.~53, no.~4, pp. 2066--2078, 2015.

\end{thebibliography}

\vfill


\end{document}